\documentclass[11pt]{article} 
\usepackage[margin=3cm]{geometry}
\usepackage[utf8]{inputenc}
\usepackage{amsmath, amssymb, amsthm, bbm, natbib, graphicx, url, algorithm2e}
\usepackage{amsfonts}
\usepackage{float}
\usepackage{graphicx}
\usepackage{bm}
\usepackage{amssymb}
\usepackage{bbm}
\usepackage{color}
\usepackage{natbib}
\usepackage{algorithmic}
\usepackage{dsfont}
\usepackage{xspace}
\usepackage{comment}

\newtheorem{theorem}{Theorem}[section]
\newtheorem{lemma}[theorem]{Lemma}
\newtheorem{definition}[theorem]{Definition}

\newtheorem{corollary}[theorem]{Corollary}

\newcommand{\total}{omnipredictor\xspace}
\newcommand{\Total}{Omnipredictor\xspace}
\newcommand{\totals}{omnipredictors\xspace}
\newcommand{\Totals}{Omnipredictors\xspace}

\newcommand{\wh}{\widehat}
\newcommand{\wt}{\widetilde}

\newcommand{\eps}{\epsilon}

\newcommand{\R}{\mathbb{R}}
\newcommand{\D}{\mathcal{D}}

\newcommand{\X}{\mathcal{X}}
\newcommand{\Y}{\mathcal{Y}}
\newcommand{\mC}{\mathcal{C}}
\newcommand{\mH}{\mathcal{H}}
\newcommand{\mS}{\mathcal{S}}

\newcommand{\mF}{\mathcal{F}}

\newcommand{\mL}{\mathcal{L}}

\newcommand{\mP}{\mathcal{P}}
\newcommand{\mI}{\mathcal{I}}

\newcommand{\mLC}{\mathrm{Lin}_\mathcal{C}}

\newcommand{\Split}{\mathrm{Split}}
\newcommand{\Merge}{\mathrm{Merge}}

\newcommand{\bll}{\bar{\ell}}

\renewcommand{\i}{\mathbf{i}}
\renewcommand{\j}{\mathbf{j}}
\newcommand{\y}{\mathbf{y}}
\newcommand{\x}{\mathbf{x}}
\newcommand{\z}{\mathbf{z}}

\newcommand{\rgta}{\rightarrow}

\newcommand{\lt}{\left}
\newcommand{\rt}{\right}

\newcommand{\zo}{\ensuremath{\{0,1\}}}
\newcommand{\izo}{\ensuremath{[0,1]}}

\newcommand{\eat}[1]{}

\newcommand{\onenorm}[1]{\left\lVert#1\right\rVert_1}
\newcommand{\infnorm}[1]{\left\lVert#1\right\rVert_\infty}

\newcommand{\macc}{multiaccuracy\xspace}

\newcommand{\Mcab}{Multicalibration\xspace}
\newcommand{\mcab}{multicalibration\xspace}
\newcommand{\mcbd}{approximately multicalibrated\xspace}
\newcommand{\mcbn}{approximate multicalibration\xspace}
\newcommand{\Mcbn}{Approximate multicalibration\xspace}
\newcommand{\smcbd}{multicalibrated\xspace}

\newcommand{\B}[1]{\mathrm{Ber(#1)}}
\newcommand{\Hspc}{\mathrm{Thr_{\mathcal{C}}}}

\renewcommand{\varepsilon}{\epsilon}
\renewcommand{\tilde}{\wt}
\renewcommand{\hat}{\wh}
\renewcommand{\eps}{\epsilon}

\newcommand{\abs}[1]{\ensuremath \Bigl\lvert #1 \Bigr\rvert}

\newcommand{\pmo}{\ensuremath{ \{\pm 1\} }}

\newcommand{\sD}{\ensuremath{(\x, \y) \sim \D}}
\newcommand{\err}{\mathrm{err}}
\newcommand{\opt}{\mathrm{OPT}}
\newcommand{\ind}[1]{\ensuremath{\mathbbm{1}(#1)}}

\newcommand{\defeq}{{:=}}
\DeclareMathOperator*{\Var}{{\bf {Var}}}
\DeclareMathOperator*{\CoVar}{{\bf {Cov}}}
\DeclareMathOperator*{\cor}{{\bf {cor}}}

\DeclareMathOperator{\poly}{poly}

\DeclareMathOperator*{\E}{\mathbf{E}}

\DeclareMathOperator*{\argmin}{arg\,min}
\DeclareMathOperator{\clip}{clip}

\newcommand{\etal}{{\it et al.}}

\newcommand{\UW}[1]{{\color{green}[Udi: #1]}}
\newcommand{\PG}[1]{{\color{red}[Parikshit: #1]}}
\newcommand{\AK}[1]{{\color{blue}[Adum: \textit{#1}]}}
\newcommand{\OR}[1]{{\color{orange}[Omer: #1]}}

\newcommand{\sign}{\mathbbm{1}^{\geq 0}}
\usepackage[utf8]{inputenc}
\usepackage[english]{babel}
\usepackage{natbib}
\usepackage{sidecap}
\begin{document}
\newcommand{\theTitle}{\Totals\ }

{
\author{Parikshit Gopalan\footnote{Email: \texttt{pgopalan@vmware.com}}\\
VMware Research
\and Adam Tauman Kalai\footnote{Email: \texttt{noreply@microsoft.com}}\\
Microsoft Research
\and Omer Reingold\footnote{Most of the work performed while visiting VMware Research. Research supported in part by
NSF Award IIS-1908774, The Simons Foundation collaboration on the theory of algorithmic fairness and The Sloan Foundation grant G-2020-13941. Email: \texttt{reingold@stanford.edu}}\\
Stanford University
\and Vatsal Sharan\footnote{Most of the work performed while at MIT. Email: \texttt{vsharan@usc.edu}}\\
USC
\and Udi Wieder\footnote{Email: \texttt{uwieder@vmware.com}}\\
VMware Research
}
}
\title{\theTitle}
\date{}

\clearpage
\maketitle

\begin{abstract}

Loss minimization is a dominant paradigm in machine learning, where a predictor is trained to minimize some loss function that depends on an uncertain event (e.g., “will it rain tomorrow?”). Different loss functions imply different learning algorithms and, at times, very different predictors. While widespread and appealing, a clear drawback of this approach is that the loss function may not be known at the time of learning, requiring the algorithm to use a best-guess loss function. Alternatively, the same classifier may be used to inform multiple decisions, which correspond to multiple loss functions, requiring multiple learning algorithms to be run on the same data. We suggest a rigorous new paradigm for loss minimization in machine learning where the loss function can be ignored at the time of learning and only be taken into account when deciding an action. 

We introduce the notion of an (${\cal L},\mC$)-\total, which could be used to optimize any loss in a family ${\cal L}$. Once the loss function is set, the outputs of the predictor can be post-processed (a simple univariate data-independent transformation of individual predictions) to do well compared with any hypothesis from the class $\cal C$. The post processing is essentially what one would perform if the outputs of the predictor were true probabilities of the uncertain events. 
In a sense, \totals\ extract all the predictive power from the class $\mC$, irrespective of the loss function in $\cal L$. 

We show that such ``loss-oblivious" learning is feasible through a connection to multicalibration, a  notion  introduced in the context of algorithmic fairness. A multicalibrated predictor doesn’t aim to minimize some loss function, but rather to make calibrated predictions, even when conditioned on inputs lying in certain sets $c$ belonging to a family $\mC$ which is weakly learnable.  We show that a $\mC$-multicalibrated predictor is also an (${\cal L},\mC$)-\total, where $\cal L$ contains all convex loss functions with some mild Lipschitz conditions. The predictors are even \totals\ with respect to sparse linear combinations of functions in $\mC$. As a corollary, we deduce that distribution-specific weak agnostic learning is complete for a large class of loss minimization tasks.

In addition, we show how multicalibration can be viewed as a solution concept for agnostic boosting, shedding new light on past results. Finally, we transfer our insights back to the context of algorithmic fairness by providing \totals\ for multi-group loss minimization.

\end{abstract}
\thispagestyle{empty}
\newpage
\tableofcontents
\thispagestyle{empty}
\newpage
\setcounter{page}{1}

\section{Introduction}

In machine learning, it is well-known that the best classifier may depend heavily on the choice of a loss function, and therefore correctly modeling the loss function is crucial for success in many applications. Modern machine learning libraries such as PyTorch, Tensorflow, and scikit-learn each offer a choice of over a dozen loss functions. However, this poses a challenge in applications where the loss is not known in advance or multiple losses may be used. For motivation, suppose you are training a binary classifier to predict whether a person has a certain medical condition ($y=1$), such as COVID-19 or some heart condition, given their attributes $x$. The cost for misclassification may vary dramatically, depending on the application. For an infectious disease, the question of whether an individual should be allowed to go out for a stroll is different than whether a person should be allowed to go to work in a nursing home. Similarly, deciding on whether to advise a daily dose of aspirin carries very different risks than recommending cardiac catheterization. Furthermore, medical interventions that may be developed in the future may carry yet other, unforeseen, risks and benefits that may require retraining with a different loss function.


We adopt the following problem setup: we are given a distribution $\D$ on $\X \times \Y$ where $\X$ is the domain and $\Y$ is the set of labels (for example $\Y$ can be $\zo$ or $[0,1]$). On each $x$, we take an action $t(x) \in \R$, and suffer loss $\ell(y,t(x))$ which depends on the label $y$ and the action $t(x)$. We will refer to the function $t:\X \rgta \R$ which maps points in the domain to actions as the hypothesis. Our goal is to find a hypothesis that minimizes $\E_\D[\ell(\y, t(\x))]$ in comparison to some reference class of hypotheses.\footnote{Such a loss function that is the expectation of a loss for individual examples is called a decomposable loss in the literature.} We will refer to a function $f$ which maps $\X$ to probability distributions over $\Y$ as a predictor. The goal of a predictor is to model the conditional distribution of labels $\y|\x =x$ for every point in the domain.

\eat{We will refer to a function $f: \X \rgta [0,1]$ as a predictor. A predictor's goal is to model the conditional probabilities $\E[$
}

A classic example of different optimal hypotheses for different loss functions is the $\ell_2$ vs.\ $\ell_1$ losses which are minimized by the mean and median respectively. Consider a joint distribution $\D$ on $\X \times [0,1]$, where $\x \in \X$ is a set of attributes (given) and $\y \in [0,1]$ is an outcome. 
Consider an $x$ such that the corresponding $y$ is uniform in $[0.8,1]$ with probability $0.6$ and 0 with probability $0.4$. In this case, to minimize the expected $\ell_2$ loss $\ell(y, t) = (y-t)^2$ you would set $t(x)=0.45$, whereas to minimize the expected $\ell_1$ loss $\ell(y, t) = |y-t|$ you would set $t(x)=0.83$. Not only is $t$ different for the two losses, you cannot learn one from the other. In other words, learning to minimize the $\ell_2$ loss looses information that is necessary to minimize the $\ell_1$ loss and vice versa.

The phenomenon that there is no simple way to get one loss-minimizing predictor from another is not unique to $\ell_2$ vs.\ $\ell_1$ losses. Consider the distribution illustrated in Figure \ref{fig:1}, which is known as a nested halfspace \cite{nested07}.  Consider a common and simple family of loss functions where $\ell(y, t) = c_y |y-t|$ where $y \in \{0, 1\}$ and $c_0, c_1 \geq 0$ are the costs of false positives and false negatives respectively. Even for linear classification, as the ratio $c_0/c_1$ varies, a different direction is optimal.  The standard ML approach of minimizing a given loss would require separate classifiers for each loss. There is no clear way to infer the optimal classifier for one set of costs from the classifier for another; applying standard post-processing techniques, such as Platt Calibration \cite{Platt99probabilisticoutputs} or Isotonic Regression \cite{isotonic01}, to the predictions so that  $\Pr[y=1|t=z] \approx z$ will not fix the issue since the optimal direction is different.

\subsection{\Totals: one predictor to rule them all}

This paper advocates a new paradigm for loss minimization: train a single predictor that could later be used for minimizing a wide range of loss functions, without having to further look at the data. Why should such predictors exist, and are they computationally tractable? One source of optimism is that the {\em ground-truth} predictor $f^*$ does allow exactly that. 

First consider the case of Boolean labels, and let $\D$ denote a distribution on $\X \times \zo$ (our results also apply to real-valued outcomes). We define $f^*(x) = \E_{\y \sim \D} [\y|x] \in [0,1]$ to be the conditional expectation of the label for $\x$. The value $t(x)$ which minimizes the loss $\E_{\y|x}[\ell(\y,t(x))]$  depends only on $f^*(x)$. Furthermore, as long as $\ell$ is smooth and easy to compute, $t$ can easily be computed from $f^*$. We denote the univariate post-processing function that optimizes loss $\ell$ given true probabilities by $k_{\ell}^*: [0,1] \rgta \R$. So for example for $\ell_2(y, t)= (y -t)^2$, we have $k_{\ell_2}^*(p)=p$ and
for $\ell_1(y,t) = |y -t|$, $k_{\ell_1}^*(p)= \ind{p \geq 1/2}$.

\begin{SCfigure}
    \includegraphics[width=2.5in]{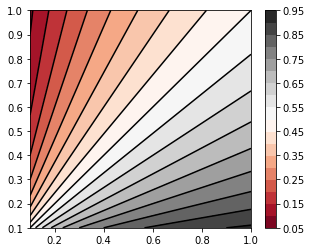}
    \caption{Binary classification with target function $\Pr[y=1|x] = \frac{x_1}{x_1+x_2}$ for $x \in [0.1,1]^2$. As can be seen from the level sets, the direction of the optimal linear classifier varies depending on the cost of false positives and negatives. This example is learned to near optimal loss for any loss with fixed costs of false-positives and false-negatives by an omnipredictor for the class ${\cal C}=\{x_1,x_2\}$.\label{fig:1}}
\end{SCfigure}

For every loss function $\ell$, the composition of $f^*$ and $k_{\ell}^*$ minimizes the loss $\ell$, even conditioned on complete knowledge of $\D$. The connection between perfect predictors and choosing the optimal action is well understood and plays an important role in the Statistics literature on proper scoring rules and forecasting (cf.\ \cite{Schervish}). But learning $f^*$ from samples from $D$ is information-theoretically and computationally impossible in general. The natural approach is to learn a model $f$ for $f^*$ and then compose $k_\ell^*$ with $f$. Common instantiations of this approach (such as using logistic regression to model $f$) do not yield particularly strong guarantees in the realistic non-realizable setting where $f^*$ does not come from the class of model distributions (see the further discussion in Section \ref{sec:related}). {\bf Our main conceptual contribution is to introduce the notion of  \Totals}, which provide a framework to derive strong rigorous guarantees using this composition approach even in the non-realizable setting.   

The goal of an \total\  is to learn a predictor $f$ that could replace $f^*$ for the purpose of minimizing any loss from a class $\mL$ compared to some hypothesis class $\mC$. For a family $\mL$ of loss functions, and a family $\mC$ of hypotheses $c: \X \rgta \R$, we introduce the notion of an {\bf (${\cal L},\mC$)-\total}, which is a predictor $f:\X \rgta \izo$ with the property that for every loss function $\ell\in{\cal L}$, there is a post-processing function $k_{\ell}$ such that the expected loss of the composition $k_{\ell} \circ f$ measured using $\ell$ is almost as small as that of the best hypothesis $c\in \mC$.

\subsection{\Totals for convex loss minimization}

\Totals can replace perfect predictors for the sake of minimizing loss in $\cal L$ compared with the class $\mC$. In this sense they {\em extract all the predictive power} of $\mC$ for such tasks. The key questions are of efficiency and simplicity: how strong a learning primitive do we need to assume to get an \total for $\mL,\mC$, and how complex is the predictor. The main result of this paper is that {\bf one can efficiently learn simple \totals for broad classes of loss functions and hypotheses, from weak learning primitives}.

\paragraph{Our main result.} We show that for any hypothesis class $\mC$, a {\em weak agnostic learner} for $\mC$ is sufficient to efficiently compute an $(\mL, \mC)$-\total where $\mL$ consists of all decomposable convex loss functions obeying mild Lipshcitz conditions; it includes popular loss functions such as the $\ell_p$ losses for all $p$, the exponential loss and the logistic loss. Weak agnostic learnability captures a common modeling assumption in practice, and is a well-studied notion in the computational learning literature \cite{SBD2, KalaiMV08, kk09, feldman2009distribution}. The weak agnostic learning assumption says that if there is a hypothesis in $\mC$ that labels the data reasonably well (say with $0$-$1$ loss of $0.7$), then we can efficiently find one that has a non-trivial advantage over random guessing (say with $0$-$1$ loss of $0.51$). In essence, our main result derives strong optimality bounds for a broad and powerful class of loss functions starting from a weak optimality condition for the $0$-$1$ loss.

Perhaps surprisingly, our results are obtained not via the machinery of convex optimization, but by drawing a connection to work on fairness in machine learning, specifically the notion of \mcab \cite{hkrr2018}. \Mcab is a notion motivated by the goal of preventing unfair treatment of protected sub-populations in prediction; it does not explicitly consider loss minimization. 
We draw a connection to \totals using  a covariance-based recasting of the notion of \mcab, that clarifies the connection of \mcab\ to the literature  on boosting \cite{kalai2005boosting, Kalai04, kk09}. A multicalibrated predictor satisfying this definition can be computed by a branching program, building on existing work in the literature on boosting \cite{MansourM2002, kalai2005boosting, Kalai04} and multicalibration \cite{gopalan2021multicalibrated}. This new connection shows that the well-known boosting by branching programs algorithms \cite{ KearnsM99, MansourM2002} yield multicalibrated predictors and can  in fact be used to derive strong guarantees for a broad family of convex loss functions.   

The post-processing function $k_{\ell}$ used in our positive results is essentially $k^*_{\ell}$ with small modifications.  As the example in Figure \ref{fig:1} demonstrates, even in natural cases, an (${\cal L},\mC$)-\total cannot be a function in $\mC$; in other words, the learning task we solve is inherently not proper. 

\paragraph{\Totals for larger classes.} An advantage of our covariance-based notion of \mcab is that it is closed under linear combinations. We use this to show that any \smcbd\ predictor  for $\mC$ is in fact an $( \mL,\mLC)$-\total where $\mLC$ consists of linear combinations of functions in $\mC$.  We give negative results for slightly larger classes, showing  that a \smcbd\ predictor  for $\mC$   is not necessarily an \total for the class $\Hspc$ which consist of thresholds of functions in $\mC$. Similarly, it is need not be an \total for the class $\mC$ but with non-convex loss functions. This shows that the connection between \mcab and \totals that we present is fairly tight.

\paragraph{\Totals for multi-group loss minimization}

A very recent line of research defines multi-group notions of loss minimization~\cite{BlumLyk20,RothblumYon21}.\footnote{The notion of loss in~\cite{RothblumYon21} is more general than in this paper and includes global functions of loss rather than the expectation of loss on individual elements. See further comparison in Section~\ref{sec:related}.} 
Multi-group loss minimization is well motivated from the point of view of fairness as it guarantees that no sub-population's loss is sacrificed for the sake of global loss minimization. Say we have a collection of sub-populations $\cal T$ and hypotheses $\cal P$ and seek to take actions such that for every set $T \in \cal T$,  our actions can compare with the best hypothesis in $P \in \cal P$ for that sub-population $T$. Indeed, one may wish to vary the loss function for various subgroups: say in a medical scenario where different age-groups are known to react differently to the same treatment. 

We derive strong multi-group loss minimization guarantees using the closure of \mcab under conditioning on subsets in $\mC$. We show that in the scenario above, a \smcbd predictor for $\mathcal{T} \times \mathcal{P} = \{ T \cdot P, T \in \mathcal{T}, P \in \mathcal{P}\}$ gives an $(\mL, \mP)$-\total for every sub-population $T \in \cal T$. Hence given a sub-population $T \in \cal T$ and a loss $\ell \in \mL$, the predictions of the \total  can be post-processed to be competitive with the best hypothesis from $\cal P$ for that loss function $\ell$ and sub-population $T$.

\paragraph{\Totals for real-valued labels.} We extend the notion of \totals\ to the setting where the labels come from an arbitrary subset $\Y \subseteq \R$. Our primary interest is in multi-class prediction where $\Y = [k]$ and the bounded real-valued setting where $\Y = \izo$. We show that \totals can be learned in this setting, for similar families of  loss functions, again assuming weak learnability of $\mC$. We also show a stronger bound for the $\ell_2$ loss than what the general theorem implies. 

\subsection{\Mcab and agnostic boosting}

One of our contributions in this work is to formalize and leverage the connections between \mcab\ and the literature on agnostic boosting. We propose a covariance-based recasting of the notion of \mcab, inspired by the literature  on boosting \cite{Kalai04, kk09}. We show that this definition has several advantages, for instance it implies some general closure properties for \mcab. In the other direction, our work suggests {\bf \mcab as a solution concept for agnostic boosting}. By specializing our main result on \totals to the $\ell_1$ loss,  we derive a new proof of the classic result of \cite{KKMS08} on agnostic learning. We elaborate on these connections in this subsection.

\paragraph{A covariance-based formulation of \mcab.} 
Calibration has been well-studied in the statistics literature in the context of forecasting \cite{Dawid}. It was introduced to the algorithmic fairness literature by \cite{KleinbergMR17}. In the setting of Boolean labels, we are given a distribution $\D$ on $\X \times \zo$ of labelled examples, and wish to learn a predictor $f: \X \rgta [0,1]$, where $f(x)$ is our model for $\E_\D[\y|x]$. The predictor $f$ is (approximately) calibrated  if for every value $v$ in its range we have that $\Pr[y=1|f(x)=v] \approx v$. This means that the prediction $f(x)$ can be interpreted as a probability that is correct in expectation over individuals in the same level set of $f$. By itself, calibration is a very weak property, both in terms of fairness as well as in terms of accuracy. 
\eat{Consider a set $T\subseteq \X$ for which half of the times $y=1$ and half of the times $y=0$. A predictor $p$ that $\forall x\in T$ sets $p(x)=1/2$ is calibrated on $T$ but has very little predictive power. Furthermore, if $T$ is a protected set, such $p$ may be the source of unfair discrimination towards those members of $T$ that see a positive outcome (or a negative outcome, depending on context). }
This motivated \cite{hkrr2018} to introduce the notion of multicalibration that asks for $f$ to be calibrated on a rich collection of subgroups $\mC$ rather than just a few protected sets.

We draw a connection to \totals by introducing a covariance-based recasting of the notion of \mcab, that clarifies the connection to the literature  on boosting 
\cite{Kalai04, kalai2005boosting, kk09}. Prior definitions consider \mcab for a family $\mC$ of sets or equivalently Boolean functions $c: \X \rgta \zo$, whereas we allow arbitrary real-valued functions. Rather than work with predictors, we define \mcab for partitions, inspired by the recent work of \cite{gopalan2021multicalibrated} in the unsupervised setting. A partition $\mS = \{S_1, \ldots, S_m \}$ of the domain is a collection of disjoint subsets whose union is $\X$. Intuitively, we want these sets to be (approximate) level sets for $f^*$. We let $\D_i$ denote the distribution $\D$ conditioned on $\x \in S_i$.  The  partition  $\mS$ gives a {\em canonical predictor} $f^\mS$ where for each $x \in S_i$, we predict $f^\mS(x) = \E_{\D_i}[\y]$. In analogy to boosting (see, e.g., \cite{Kalai04, kalai2005boosting, kk09}), we phrase  \mcab  in terms of the covariance\footnote{Recall that $\CoVar[\z_1\z_2] = \E[\z_1\z_2] - \E[\z_1]\E[\z_2]$} between $c(\x)$ and $\y$ conditioned on each state $S_i$ of the partition. We  say that $\mS$ is $\alpha$-\smcbd for $\mC$ if for every $i \in [m]$ and $c \in \mC$,  
\begin{align}
\label{eq:m-cal-int}
    \abs{\CoVar_{\D_i}[c(\x), \y]} \leq \alpha. 
\end{align}
In reality, we will weaken the definition to only hold in expectation under $\D$ rather than require it for every state of the partition, but we ignore this distinction for now. While our definition is formulated differently, we show that it matches the original definition when $\mC$ consists of Boolean functions (see Lemma \ref{lem:eq-boolean} and the discussion in Appendix \ref{app:def-mcab}). This lets us adapt existing algorithms in the literature \cite{MansourM2002, Kalai04, gopalan2021multicalibrated} to give an efficient procedure to compute multicalibrated partitions, assuming  a weak agnostic learner for the class $\mC$. Working with covariance which is bilinear lets us derive powerful closure properties for \mcab. For instance, if $f$ is multicalibrated with respect to $\mC$ it is also multicalibrated (with some deterioration in parameters) with respect to the class $\mLC$ of (sparse) linear combinations of $\mC$. We also show that conditioning on sets in $\mC$ preserves \mcab.

\paragraph{Agnostic boosting from \mcab.}

The problem of agnostically learning a class $\mC$ is, given samples from a distribution $\D$ on $\X \times \zo$, find a binary classifier $f:\X \rgta \zo$ whose classification error aka $0$-$1$ loss defined as $\err(f) = \Pr_{\D}[f(\x) \neq \y]$ is not much larger than that of the best classifier from $\mC$. The boosting approach to agnostic learning is to start from a weak agnostic learner, which only guarantees some non-trivial correlation with the labels, and boost it to obtain a classifier that agnostically learns $\mC$ \cite{SBD2, KalaiMV08, feldman2009distribution}.

Our work suggests \mcab as a solution concept for agnostic boosting. Indeed, our definition of \mcab based on covariance parallels that of \cite{Kalai04, kalai2005boosting}, who use covariance as a splitting criterion. Hence when the algorithms of \cite{MansourM2002, Kalai04, kalai2005boosting} terminate, they have found a \smcbd partition. Viewed in this light, our results show that these algorithms give a broad and powerful guarantee beyond just  $0$-$1$ loss: they are competitive with sparse linear combinations over $\mC$ in optimizing a large family of convex, Lipschitz loss functions (with a simple post-processing step). While AdaBoost or Logistic regression are known to minimize the exponential and logistic loss respectively over sparse linear combinations of $\mC$ \cite{SSbook}, no similar result was known for algorithms based on branching programs \cite{MansourM2002, Kalai04, kalai2005boosting}. 

Let us now focus on $0$-$1$ loss. Since $0$-$1$ loss for Boolean functions equals $\ell_1$ loss, our results apply to it. We show that for any \smcbd partiton $\mS$, the  predictor $k_{\ell_1} \circ f^\mS$ is competitive not just with the best classifier in $\mC$, but with the best classifier in the larger class $\mH$ of functions that are approximated (in $\ell_1$) by linear combinations of $\mC$. This lets us re-derive the classic result of \cite{KKMS08} on agnostic learning. While not a new result, we feel that our treatment clarifies and unifies existing results (see Theorem \ref{thm:ag} and the comments following it). It strengthens known results on the noise-tolerant boosting abilities of such programs \cite{kalai2005boosting, KalaiMV08}.
 Further, we show examples where $\err(k_{\ell_1} \circ f^\mS)$ is markedly better than any linear combination of $\mC$, showing that \mcab is a stronger solution concept than those considered previously.

\subsection{Technical overview: \Totals from \mcab}

Let $\ell: \zo \times \R \rgta \R$ be a loss function that takes a label $y \in \Y$ and an action $t \in \R$ as arguments. A hypothesis $t: \X \rgta \R$ which prescribes an action for every point in the domain suffers a loss of $E_\D[\ell(\y, t(\x)]$. Our main technical result states that if $\mS$ is \smcbd, then $f^\mS$ is an $(\mL,\mC)$-\total where $\mL$ consists of all convex losses satisfying some mild Lipschitz conditions. Here for simplicity, we make the stronger assumption that the loss $\ell(y, t)$ is convex and Lipschitz everywhere as a function of $t$. We do not assume anything about the relation between $\ell(0,t)$ and $\ell(1,t)$. 

We sketch how \mcab leads $f^\mS$ to be an \total, emphasizing intuition over rigor. We will argue that the loss $\E_\D[\ell(\y, k^*_\ell(f^\mS(\x))]$ is not much more than $\E_\D[\ell(\y,c(\x))]$ for any $c \in \mC$. 
We fix a state $S_i \in \mS$ and analyze the loss suffered by $c(x)$ under $\D_i$ as follows.
\begin{enumerate}
    \item {\bf Reduction to predicting two values:} In general, $c(x)$ could take on many values under $\D_i$. However, since the goal is to minimize  $\E_{\D_i}[\ell(\y, c(\x))]$, we can {\em pretend} that $c$ takes only two values, $\E_{\D_i|\y =0}[c(\x)]$ whenever $\y =0$ and $\E_{\D_i|\y = 1}[c(\x)]$ whenever $\y =1$ (we say pretend since the actions taken can only depend on $x$ and not on $y$). By the {\em convexity} of the loss functions, this can only reduce the expected loss. 
    
    \item {\bf Reduction to predicting one value: } A consequence of \mcab, which follows from the definition of covariance, is that  conditioning on the label $\y =b$  does not change the expectation of $c(\x)$ much. Formally for $b \in \zo$, 
    \begin{align}
    \Pr_{\D_i}[\y = b]\lt|\E_{\D_i|\y = b}[c(\x)] - \E_{\D_i}[c(\x)]\rt| \leq \alpha. \label{eq:multicalibration_corollary_int}
    \end{align}
    Since the loss functions $\ell(b, t)$ are {\em Lipschitz} in $t$ for $b \in \zo$, we can replace $\E_{\D_i|\y = b}[c(\x)]$ with $\E_{\D_i}[c(\x)]$ with only a small increase in the loss. At this point, we have reduced to the case where $c$ predicts the constant value $\E_{\D_i}[c(\x)]$ under $\D_i$. 
    
    \item {\bf The best value: } Let $\E_{\D_i}[\y] = p_i$, thus $\y$ is distributed as a Bernoulli random variable with parameter $p_i$.    
    Thus, the best single value to predict is $k^*_\ell(p_i)$, which is the minimizer of the expected loss $p_i \ell(0, t)  + (1 - p_i)\ell(1, t)$. But we defined $f^\mS(x) = p_i$ for all $x \in S_i$, so $k^*_\ell(p_i) = k^*_\ell(f^\mS(x)$).
\end{enumerate}
We conclude that post-processing the predictions $f^\mS$ by the function $k^*_\ell$ is nearly as good as any $c \in \mC$ for minimizing expected loss under $\D$ for any convex, Lipschitz loss function $\ell$ (up to an additive error that goes to $0$ with $\alpha$). Hence $f^\mS$ is an $(\mL, \mC)$-\total. As a consequence, having a weak agnostic learner for $\mC$ suffices to learn a predictor that can minimize any such loss function competitively to predictors in $\mC$, even without knowing the loss function in advance.

\eat{
Our {\bf main result} is that {\em every $\tilde{f}$ that is multicalibrated with respect to $\mC$ is also $({\cal L},\mC)$-\total, where ${\cal L}$ contains all convex loss functions with some mild Lipshitz condition} (See Theorem~\ref{} for a formal statement). 

\UW{Here we switched from multi calibrated partitions to multi calibrated predictors. Instead of saying every $\tilde{f}$ we should say every canonical predictor of a multi-calibrated partition.}

The main result implies that loss-oblivious loss-minimization through \totals is feasible. In addition, we show the following:
\begin{itemize}
    \item Using the aforementioned closure properties for multicalibration, we obtain that multicalibration with respect to $\mC$ implies \total with respect to sparse linear combinations of functions from $\mC$.
    \item As a corollary of our main result, loss minimization for a wide family of loss functions reduces to distribution-specific weak-agnostic learning.  
    \item We extend our results to real-valued outcomes by generalizing \totals to this context and showing that multicalibration for real-valued outcomes (which can be learned via~\cite{Jung20}) implies \totals.
    \item We look more closely at $\ell_2$ and $\ell_1$ losses and give stronger results than for generic loss functions. In particular, we demonstrate that multicalibration can be strictly better than every function in $\mC$. 
\end{itemize}
\PG{Say more about results in the agnostic setting}}

\subsection{Organization of this paper}
We survey related work in Section \ref{sec:related}, and set up notation in Section \ref{sec:notation}. We define the notion of \totals in Section \ref{sec:omni}. We introduce our notion of \mcab in Section \ref{sec:mcab} and derive closure properties for it in Section \ref{sec:closure}. A detailed discussion of how our definition compares to previous definitions can be found in Appendix \ref{app:def-mcab}. We prove our main result on $(\mL, \mC)$-\totals for binary labels (Theorem \ref{thm:key}) in Section \ref{sec:binary}, along with the  extensions to $\mLC$ (Corollary \ref{cor:lipschitz}), and its application to  multi-group loss minimization (Corollary \ref{cor:sub_pop}). Section \ref{sec:limits} shows that \mcab for $\mC$ does not yield \totals for thresholds of $\mC$ or for non-convex losses. Section \ref{sec:ag} presents applications of our results to agnostic learning, and an example showing that \mcab can give stronger guarantee than $\opt(\mLC)$. We present the extension to the real valued setting in Section \ref{sec:real}, and derive a stronger bound for $\ell_2$ loss in Theorem \ref{thm:ell_2}. We present an algorithm for computing multicalibrated partitions in Section \ref{sec:algo}.

\eat{
The guarantees we present for \totals are reminiscent of agnostic boosting \cite{KearnsM99, MansourM2002, Kalai04} in the sense that they apply to a larger class ${\cal H} \supset {\cal C}$ using a learner for ${\cal C}$. Agnostic boosting guarantees a 0-1 loss $|y-t|$ close to the optimal in ${\cal H}$. In fact, as we point out in Section \ref{sec:boosting}, \total guarantees are stronger. First, they imply optimality with respect to ${\cal H}$ just as agnostic boosting algorithms do. Second, we give a simple example illustrating where \totals guarantee loss below the best in ${\cal H}$. Third, of course \totals apply to more general families of loss functions than 0-1 loss.

Note that noiseless boosting (e.g., AdaBoost \cite{freund1997decision}) and PAC learning, more generally, are not interesting for \totals as in such a ``realizable'' setting one can achieve $\eps$ error, which implies low loss with respect to all bounded loss functions. 
}

\eat{
\subsection{Our Contributions}

The main contribution of this paper is introducing the notion of \totals as an effective method of learning for loss-minimization that doesn't depend on the loss function at the time of learning. We show how to implement \totals through a connection to multicalibrated predictors. We conclude that distribution-specific weak-agnostic learning is sufficient for minimizing every convex loss-function that satisfies some mild Lipschitz conditions.  Beyond these contributions, we make a sequence of additional, technical and conceptual contributions, including:
\begin{itemize}
\item We give a new formulation of multicalibration that identifies with the original definition in the basic setting but generalizes well to real valued functions $c$ (instead of the indicators of sets) and to real-valued outcomes. This characterization also allows us to prove a useful closure properties with respect to sparse linear combinations and with respect to subsets.
\item We give stronger results for two of the most important loss functions: $\ell_2$ and $\ell_1$ (classification loss), showing that at times multicalibration can do strictly better than any function in $\mC$.
\item We demonstrate that multicalibration is a strong solution concept for agnostic boosting by {\bf clarifying/ simplifying/ strengthening} previous work in the area. {\bf be more specific?}
\item We show how to provide \totals that allow for the simultaneous minimization of the loss function on a large collection of sub-groups, without knowing the exact loss functions a-priori.
\end{itemize}
}

\section{Related work}
\label{sec:related}

While the notion of an \total is introduced in this work, our definitions draw on two previous lines of work. The first is the notion of \mcab\ for predictors that was defined in the work of Herbert-Johnson \etal\ \cite{hkrr2018}.\footnote{See also \cite{kearns2018}, who in parallel with \cite{hkrr2018} introduced notions of multi-group fairness.} A detailed discussion of how our definitions of \mcab compare to previous definitions appears in Appendix \ref{app:def-mcab}.
The other is work on boosting by branching programs of  Mansour-McAllester \cite{MansourM2002}, which built on Kearns-Mansour \cite{KearnsM99} and the notion of correlation boosting \cite{Kalai04, kalai2005boosting}.

\paragraph{Group fairness and \mcab.}
While \mcab  was introduced with the motivation of algorithmic fairness, it has been shown to be quite useful from the context of accuracy when learning in an heterogeneous environment. This was done both experimentally \cite{kgz,BardaYRGLBBD21} and in real-life implementations~\cite{BardaEtal20}. From a theoretical perspective, it has been shown in \cite{hkrr2018} that post-processing a predictor to make it multicalibrated cannot increase the $\ell_2$ loss. This was extended to showing some optimality results of multicalibrated predictors with respect to $\ell_2$ and even log-loss \cite{GargKR19,KimThesis}. Multicalibrated predictors are also connected to loss minimization through the notion of outcome-indistinguishability \cite{OI} in the work of \cite{RothblumYon21}. Outcome indistinguishability shows that a multicalibrated predictor is indistinguishable from the true probabilities predictor in a particular technical sense. In~\cite{RothblumYon21} this is used to create a predictor that can be used to minimize a rather general and potentially global notion of loss even when restricted to sub-groups. The proof constructs a family of distinguishers, for a fixed loss function, such that if there exists a subset on which the a predictor doesn't minimize the loss function then one of the predictors can distinguish the predictor from the true-probabilities predictor in the sense of outcome indistinguishability. The main way in which all of these results are different from what we show is that they do not seek to simultaneously allow for the minimization of such a rich family of loss functions. In the case of~\cite{RothblumYon21}, since the result goes through outcome indistinguishability, to minimize a loss with respect to a class $\mC$, a predictor needs to be multicalibrated with respect to a different class that relies on the loss function and incorporates the reduction from multicalibration and outcome indistinguishability. In contrast, our result addresses minimization of a  family of losses that may not be known at the time of learning, and we assume the learnability of $\mC$ alone. Convexity of the losses plays a key role in our upper bounds.

\paragraph{Agnostic boosting.} 
Our work is  closely related to work on boosting via branching programs \cite{KearnsM99, MansourM2002, kalai2005boosting, Kalai04}. Indeed, the splitting criterion used in the work of Kalai \cite{Kalai04} is precisely that $\CoVar[c(\x), \y] \geq \alpha$ for some $c \in \mC$, for the task of learning generalized linear models. The {\em okay learners} in the work of \cite{kalai2005boosting} are also based on having non-trivial Covariance with the target.  Our results show that upon termination, these algorithms yields an \mcbd\ partition; hence we can view \mcab\ as a solution concept for the output of Boosting by Branching Program family of algorithms. It was known that these algorithms have stronger noise tolerance properties than potential based algorithms such as AdaBoost, see \cite{kalai2005boosting, KalaiMV08}. Our results significantly strengthen our understanding of the power of these algorithms, showing that they give guarantees for a broad family of convex loss functions, and not just $0$-$1$ loss. 

\paragraph{Naive instantiations of the composition approach.}
The obvious attempt to building \totals would be to learn a model $f$ for $f^*$ using a model family $\mF$ such as logistic regression, and then compose it with the right post-processing function $k_{\ell}^*$ for a given loss $\ell$. In the context of binary classification, for certain families of non-decomposable accuracy metrics (including $F$-scores and AUC), it is shown in \cite{NatarajanKRD15, DKKN17} that this approach gives the best predictor for the hypothesis class $\mC$ of binary classifiers derived by thresholding models from $\mF$. 

These results can be seen as a form of \totals, but with strong restrictions on $\mL$ and $\mC$. Their results do not apply to the decomposable convex losses we consider; indeed it seems unlikely that the output of logistic regression can give reasonable guarantees for say the exponential loss or squared loss, even with post-processing. More importantly,  our results hold for arbitrary classes $\mC$ that are weakly agnostically learnable. For any such class, we show how to construct a model $f$ which is an \total. In contrast, their results prove optimality for a rather limited hypothesis class $\mC$ derived from the model family $\mF$.

\paragraph{Condtional density estimation}
For real-valued $y \in \R$, an \total solves the problem of \textit{Conditional Density Estimation} (CDE) \cite{hansen1994autoregressive}. While CDE is recognized as an important problem in practice and a number of CDE algorithms have been proposed, it has received little attention in the computational learning theory literature. The notion of \total is related to the statistical notion of a \textit{sufficient statistic}, which is a statistic that captures all relevant information about a distribution. 

\paragraph{Surrogate loss functions for classification} There is a large literature in statistics which shows that a convex surrogate loss function (with certain properties) can be used instead of the hard to optimize 0-1 loss, and any hypothesis $c \in \mC$ which minimizes the surrogate loss will also minimize the original 0-1 loss with respect to $\mC$ \cite{lugosi2004bayes,steinwart2005consistency,bartlett2006convexity}. There are also similar results for the multi-class 0-1 loss \cite{tewari2007consistency}, and in the asymmetric setting when the false positive and false negative costs are known \cite{scott2012calibrated}. However, this line of work is quite different from ours, crucially an $(\mathcal{L}, \mC)$-\total optimizes not just for a single loss but for any loss in the family $\mathcal{L}$ (such as different false positive and false negative costs, we recall that as Fig. \ref{fig:1} shows this cannot be achieved with a single hypothesis from $\mC$). 


\section{Notation and Preliminaries}
\label{sec:notation}

Let $\D$ denote a distribution on $\X \times \Y$. The set $\X$ represents points in our space, it could be continuous or discrete. The set $\Y$ represents the labels, we will typically consider $\Y = \zo$, $\Y = [k]$ or $\Y = \izo$. We use  $(\x, \y) \sim \D$ where $\x \in \X, \y \in \Y$ to denote a sample from $\D$.  We use boldface for random variables. For any $S \subset \X$, let $\D(S) = \Pr_{\x \sim \D}[ \x \in S]$. We will use $\y \sim \B{p}$ to denote sampling from the Bernoulli distribution with parameter $p$. 

Let $\mC = \{c: \X \rgta \R\}$ be a collection of real-valued hypotheses on a domain $\X$, which could be continuous or discrete. The hypotheses in $\mC$ should be efficiently computable, and the reader can think of them as monomials, decision trees or neural nets. We will denote
\[ \infnorm{\mC} = \max_{c \in \mC, x \in \X}|c(x)|. \]

A loss function $\ell$ takes a label $y \in \Y$, an action $t \in \R$ and returns a loss value $\ell(y,t)$.  Common examples are the $\ell_p$ losses $\ell_p(y, t) = |y -t|^p$ and logistic loss $\ell(y, t) =  \log(1 + \exp(-yt))$. The problem of minimizing a loss function $\ell$ is to learn a hypotheses $h: \X \rgta \R$ such that the expected loss $\ell_\D(h) \defeq \E_\D[\ell(\y, h(\x))]$ is small. Let $\mL = \{\ell: \Y \times \R \rgta \R\}$ denote a collection of loss functions.

A partition $\mS = \{S_1, \ldots, S_m \}$ of the  domain $\X$, is a collection of disjoint subsets whose union equals $\X$, we refer to $m$ as its size. We refer to each $S_i$ as a state in the partition. Given  a partition $\mS = \{S_1, \ldots, S_m \}$ of the  domain $\X$, define the conditional distribution $\D_i$ over $S_i \times \Y$ as $\D_i = \D|\x \in S_i$. 

A (binary) predictor is a function $f:\X \rgta \izo$, where $f(x)$ is interpreted as the probability conditioned on $x$ that $\y=1$. We define the ground truth predictor as $f^*(x) := {\E_\D[\y|\x =x]}$.  For general label sets $\Y$,  let $\mP(\Y)$ denote the space of probability distributions on $\Y$. Define the ground truth predictor $f^*: \X \rgta \mP(\Y)$ where $f^*(x)$ is the distribution of $\y|x$. A predictor is a function $f:\X \rgta \mP(\Y)$ which is intended to be an approximation of $f^*$.

 We denote by $g \circ h$ the composition of functions.
 
 \subsection{Nice loss functions}
We say $\ell:\Y \times \R \rightarrow \R$ is a convex loss function,  if $\ell(y, t)$ is a convex function of $t$ for every $y \in \Y$. Note that in the binary setting, our formulation allows for binary classification with different false-positive/negative costs, e.g.,  $\ell(y, t) = c_y |t-y|$ where $c_{0}$ and $c_1$ are the different costs. 
A function $f:\R\rightarrow \R$ is said to be $B$-Lipschitz over interval $I$ if $|f(t)-f(t')|\leq B|t-t'|$ for all $t, t' \in I$. We say the function is $B$-Lipschitz if the condition holds for $I =\R$. While assuming the loss is $B$-Lipschitz is sufficient for us, we can work with a weaker notion that only requires the Lipschitz property on a sufficiently large interval. This weaker notion covers most commonly used loss functions such as the exponential loss that are not Lipschitz everywhere. Also, we will define loss functions in the setting of labels that come from $\Y$. In this section though, we will focus on the case $\Y = \zo$.

\begin{definition}
\label{def:nice-loss}
For $B, \eps >0$, a convex loss function $\ell: \Y \times \R \rgta \R$ is $(B, \eps)$-nice if there is a closed interval $I = I_\ell \subseteq \R$  satisfying:
\begin{enumerate}
\item {\bf ($B$-Lipschitzness)} For all $y$, $\ell(y, t)$ is $B$-Lipschitz in $t$ over $I_{\ell}$: 
\[ \forall y \in \Y, \ t, t' \in I_{\ell},  |\ell(y, t) - \ell(y, t')| \leq B |t - t'|.\] 
\item {\bf ($\eps$-optimality)} For $y \in \Y$, $I_{\ell}$ contains an $\eps$-optimal minimizer of $\ell(y,t)$: 
\[ \inf_{t \in I_{\ell}}\ell(y, t) \leq \inf_{t \in \R} \ell(y, t)+\eps.\] 
\end{enumerate}
Let $\mL(B, \eps)$ be the set of all $(B, \eps)$-nice functions.
\end{definition}
 Observe that if a function is $B$-Lipschitz on $\R$, then it is $(B, 0)$-nice with $I_\ell =\R$. 
 
 For a closed interval $I = [c,d]$ define the function
\begin{align*}
    \clip(t, I) = \begin{cases}
    c \ \text{if} \ t \leq c\\
    t \ \text{if} \ c \leq t \leq d\\
    d \ \text{if} \ d \leq  t.
    \end{cases}
\end{align*}
We use some simple facts about this function without proof, the  first is a simple consequence of $\eps$-optimality and convexity.
\begin{lemma}
\label{lem:tech}
If $\ell$ is $(B, \eps)$-nice, then $\ell(y,\clip(t, I_\ell)) \leq \ell(y,t)  + \eps$. The function $\clip(t,I_\ell)$ is $1$-Lipschitz as a function of $t$.
\end{lemma}

Here are a few examples of nice loss functions:
\begin{itemize}
    \item Binary classification with different false-positive/negative costs, e.g.,  $\ell(y, t) = \kappa_y |y-t|$ where $\kappa_0 \neq \kappa_1$ are the different costs. Here $I_\ell =[0,1], B =\max(\kappa_0, \kappa_1), \eps=0$.
    \item The $\ell_p$ losses for $p \geq 1$:
    \[ \ell_p(y, t) \defeq |y-t|^p \]
    take $I_\ell=[0,1], B=p, \eps=0$. 
    \item The exponential loss $\ell(y, t) = e^{(1 - 2y) t}$. Here, for any $\eps > 0$, we can take $I_\ell=[-\ln (1/\eps), \ln (1/\eps)]$ and $B = 1/\eps$.
    \item The logistic loss $\ell(y, t) = \log(1 + \exp((1 -2y)t))$. Here we take $I_\ell = \R$, $B =1$ and $\eps = 0$. 
    \item The hinge loss $\ell(y, t)=\max(0,1+(1-2y)t)$. Here we take $I_\ell = \R$, $B =1$ and $\eps = 0$. 
\end{itemize}
It is worth noting that only for exponential loss did we need $\eps > 0$. 

\section{\Totals}
\label{sec:omni}

In this section, we define our notion of \Totals. Our  definitions are simpler for the case of binary labels where $\Y = \zo$, hence we present that case first.
Recall that for a predictor $h$, $\ell_\D(h)$ is the expected loss of $h$ under $\D$. 

\begin{definition}[\Total]
\label{def:omni}
    Let $\mC$ be  family of functions on $\X$, and let $\mL$ be a family of loss functions. The predictor $f:\X \rgta \izo$ is an $(\mL, \mC, \delta)$-\total if for every $\ell \in \mL$ there exists a function $k:\izo \rgta \R$ so that
    \[ \ell_\D(k \circ f) \leq \min_{c \in \mC} \ell_\D(c) + \delta. \]
\end{definition}

The definition states that for every loss $\ell$, there is a simple (univariate) transformation $k$ of the predictions $f$, such that the composition $k \circ f$ has loss comparable to the best hypothesis $c \in \mC$, which is chosen tailored to the loss $\ell$.

Setting aside efficiency considerations, it is easy to show that $f^*$ is an \total for every $\mC, \mL$.

\begin{lemma}
    For every $\mC, \mL$, the ground-truth predictor $f^*$ is an $(\mL, \mC, 0)$-\total.
\end{lemma}
\begin{proof}
By the definition of \totals, given an arbitrary loss function $\ell: \zo \times \R \rgta \R$, our goal is to find $k^*_\ell:\zo \rgta \R$ so that 
\begin{align}
\label{eq:f*}
\ell_\D(k^*_\ell \circ f) \leq \min_{c:\X \rgta \R} \ell_\D(c). 
\end{align}
 
 Define the function $k^*_\ell: \izo \rgta \R$ which minimizes expected loss under the Bernoulli distribution:
\begin{align}
\label{eq:def-k*}
    k_\ell^*(p) = \argmin_{t \in \R} \E_{\y \sim \B{p}}[\ell(\y,t)] = \argmin_{t \in \R} p\ell(1, t) +  (1 - p)\ell(0, t).
\end{align} 
If there are multiple minima we break ties arbitrarily. Conditioned on $\x =x$, $\y \sim \B{f^*(x)}$, so $k_\ell^*(f^*(x)) \in \R$ is the value that minimizes the expected loss. Hence for every $x \in \X$,
\[ \E_{\D|\x =x}[\ell(\y, k_\ell^*(f^*(x))] \leq \E_{\D|\x =x}[\ell(\y, c(x))].\]
Equation \eqref{eq:f*} follows by averaging over all values of $x$.
\end{proof}

Note that the function $k^*_\ell$ depends on $\ell$ but is independent of the distribution $\D$. For instance, for the $\ell_1$ loss, $\ell_1(y, t) = |y -t|$, we have $k_{\ell_1}^*(p) = \ind{p \geq 1/2}$. For the $\ell_2$ loss  $\ell_2(y,t) = (y -t)^2$, we have $k_{\ell_2}^*(p) = p$. While Definition \ref{def:omni} does not place any restrictions on the post-processing function $k$, in our upper bounds, we will choose $k$ which is very close to the function $k^*$ above. Our upper bounds will be efficient for (convex) functions $\ell$ such that a good approximation to $k^*$ can be be approximated efficiently. Our lower bounds will hold for arbitrary $k$.  

Finally, a natural family of predictors arising from partitions plays a key role in our results. 
\begin{definition}
\label{def:canonical}
     Given a partition $\mS$ of $\X$ of size $m$, let $\E_{\D_i}[\y] = p_i \in [0,1]$ for $i \in [m]$. The canonical predictor for $\mS$ is $f^\mS(x) = p_i$ for all $x \in S_i$. 
\end{definition}
The canonical predictor simply predicts the expected label in each state of the partition. Since $f^\mS$ is constant within each state of the partition, it can be viewed as a function $f^\mS: \mS \rgta \R$. This view will be useful in our results.

\paragraph{\Totals for general $\Y$.}

Consider the setting where we are given a distribution $\D$ on $\X \times \Y$ for $\Y \subseteq \R$, hence the labels can take on real values. We are primarily interested in the real-valued setting of $\Y = \izo$, and the multi-class setting where $\Y =[l]$. 

Given a state $S_i$ in a partition $\mS$, let   $P_i \in \mP(\Y)$ denote the distribution of $\y$ under $\D_i$. The canonical predictor $f^\mS: \X \rgta \mP(\Y)$ is given by $f^\mS(x) = P_i$ for all $x \in S_i$. 

We now define the notion of an \total. The main difference from the binary case is that the predictor now predicts a distribution in $\mP(\Y)$, so the post-processing function $k$ maps {\em distributions} to real values.

\begin{definition}
\label{def:omni-real}
    Let $\mC$ be  family of functions on $\X$, and let $\mL$ be a family of loss functions $\ell: \Y \times \R \rgta \R$. The predictor $f:\X \rgta \mP(\Y)$ is an $(\mL, \mC, \delta)$-\total if for every $\ell \in \mL$ there exists a function $k:\mP(\Y) \rgta \R$ so that
    \[ \ell_\D(k \circ f) \leq \min_{c \in \mC} \ell_\D(c) + \delta. \]
\end{definition}

\eat{
An equivalent definition can be given in terms of the conditional expectations.
\begin{definition}
\label{def:m-cal-01}
The partition $\mS$ of $\X$ is $\alpha$-\smcbd for $\mC, \D$ if for every $i \in [m]$ such that $\Var_{
\D_i}[y] > 0$ and for every $c \in \mC$,
\begin{align}
\label{eq:m-cal-01}
    \abs{\E_{\D_i|\y =1}[c(\x)]  - \E_{\D_i|\y = 0}[c(\x)]}\leq \frac{\alpha}{\Var_{\D_i}[y]} . 
\end{align}
\end{definition}

Having $\Var_{\D_i}[\y] > 0$ ensures that both the conditional distributions $\D_i|\y =1$ and $\D_i|\y = 0$ are well defined. For  states where $\y$ is constant and $\Var_{\D_i}[\y] =0$, we say that the \mcab\ condition holds trivially.

\begin{lemma}
    Definitions \ref{def:m-cal} and \ref{def:m-cal-01} are equivalent. 
\end{lemma}
\begin{proof}
    First consider states where $\y \sim \D_i$ is constant, hence $\Var_{\D_i}[\y] = 0$. Such states are trivially \mcbd by Definition \ref{def:m-cal-01}. They also satisfy Equation \eqref{eq:m-cal} since
    \[ \CoVar_{\D_i}[c(\x), \y] = \E_{\D_i}[c(\x)(\y - \E_{\D_i}[\y])] = 0.\]

    Now consider states where $\y$ is not constant.
    We can simplify Equation \eqref{eq:m-cal} by putting $p_i = \E_{\D_i}[\y]$, and rewriting it as    
    \begin{align}
    \label{eq:m-cal2}
        \abs{\CoVar_{\D_i}[c(\x), \y]} = \abs{\E_{\D_i}[c(\x)(\y - p_i)]}  \leq \alpha .
    \end{align}

    We can sample $(\x, \y)$ from $\D_i$ in two steps, we first sample $\y \in \zo$ where $\Pr[\y =1] = p_i$, and then sample $\x \sim D_i|\y$.
    Hence we can write
    \begin{align*}
        \E_{(\x, \y) \sim \D_i}[c(\x) (\y - p_i)] &= p_i\E_{\x \sim D_i|\y = 1}[c(\x) (1 - p_i)] + (1 -p_i) \E_{\x \sim D_i|\y = 0}[c(\x) ( -p_i)]\\
        &= p_i(1 -p _i) \left(\E_{\D_i|\y =1}[c(\x)]  - \E_{\D_i|\y = 0}[c(\x)]\right)\\
        &= \Var_{\D_i}[\y] \left(\E_{\D_i|\y =1}[c(\x)]  - \E_{\D_i|\y = 0}[c(\x)]\right)
    \end{align*}
    Hence Equation \eqref{eq:m-cal2} and \eqref{eq:m-cal-01} are equivalent assuming $\Var_{\D_i}[\y] \neq 0$.
\end{proof}

\subsection{A relaxed definition}
}

\section{\Mcab}
\label{sec:mcab}

In this section, we present our definitions of multicalibration. We first consider the case of binary labels, and then extend it to the multi-class and real-valued settings.  To streamline the presentation, routine proofs  are deferred to Appendix \ref{app:mcab}. Given real-valued random variables $\z_1, \z_2$ from a joint distribution $\D$, we define
\begin{align*}
     \CoVar_\D[\z_1, \z_2]  &=  \E_\D[\z_1\z_2] - \E_\D[\z_1]\E_\D[\z_2] = \E_\D\left[\z_1(\z_2 - \E_\D[\z_2])\right]. 
\end{align*}
We will use the fact that covariance is bilinear. 
\eat{linear in each argument, hence
\begin{align*}
     \CoVar_\D[\lambda\z_1 + \lambda'\z_1', \z_2]  &= \lambda \CoVar_\D[\z_1, \z_2]  + \lambda'\CoVar_\D[\z_1', \z_2]  
\end{align*}
    }
The following identity will be useful for Boolean $\y$.

\begin{corollary}
\label{cor:covar}
For random variables $(\z, \y) \sim \D$ where $\y \in \zo$,
\begin{align}
\label{eq:covar-1}    
\CoVar_\D[\z, \y] = \Pr_{\D}[\y = 1]\lt(\E_{\D|\y = 1}[\z] - \E_{\D}[\z] \rt) = \Pr_{\D}[\y = 0]\lt(\E_{\D}[\z] - \E_{\D|\y = 0}[\z]\rt).
\end{align}
\end{corollary}

\subsection{\Mcab via covariance}

In this section, we define \mcab for the binary labels setting where $\Y = \zo$. We build on a recent line of work \cite{hkrr2018, Jung20, kgz, gopalan2021multicalibrated}. A detailed discussion of how our definitions compare to previous definitions is presented in Appendix \ref{app:def-mcab}. The following definition is  a generalization of the notion of $\alpha$-\mcab to real-valued $c$, and in the setting of partitions. 

\begin{definition}
\label{def:m-cal}
Let $\D$ be a distribution on $\X \times \zo$. The partition $\mS$ of $\X$ is $\alpha$-\smcbd for $\mC, \D$ if for every $i \in [m]$ and $c \in \mC$, the conditional distribution $\D_i = \D|\x \in S_i$ satisfies
\begin{align}
\label{eq:m-cal}
    \abs{\CoVar_{\D_i}[c(\x), \y]} \leq \alpha. 
\end{align}
\end{definition}
A consequence of this definition is that for each $\D_i$, conditioning on $\y$ does not change the expectation of $c(\x)$ by much. 
Formally, by Equation \eqref{eq:covar-1}, for $i \in [m]$ and $b \in \zo$, 
\begin{align}
    \Pr_{\D_i}[\y = b]\lt|\E_{\D_i|\y = b}[c(\x)] - \E_{\D_i}[c(\x)]\rt| \leq \alpha. \label{eq:multicalibration_corollary}
\end{align}
\eat{In contrast, \cite{hkrr2018} consider  $c: \X \rgta \zo$ which can be thought of as the indicator of a set, and require that for each $\D_i$, conditioning on being in the set does not change the expectation of $\y$. }

 Definition \ref{def:m-cal} requires a bound on the covariance for every distribution $\D_i$. This might be hard to achieve if $\D(S_i)$ is tiny, and hence we hardly see samples from $\D_i$ when sampling from $\D$. This motivated a relaxed definition called $(\alpha, \beta)$-\mcab in \cite{hkrr2018, gopalan2021multicalibrated}.  We propose a different definition for which it is also easy to achieve sample efficiency. Rather than requiring the covariance be small for every $i$, we only require it to be small on average. Let $\i \sim \D$ denote sampling (the index of) a set from the partition $\mS$ according to $\D$ so that $\Pr[\i = i] = \D(S_i)$.

\begin{definition}
\label{def:m-cal-1}
The partition $\mS$ of $\X$ is $\alpha$-\mcbd for $\mC, \D$ if for every  $c \in \mC$, 
\begin{align}
\label{eq:m-cal-new}
    \E_{\i \sim \D}\abs{\CoVar_{\D_\i}[c(\x), \y]} \leq \alpha. 
\end{align}
\end{definition}

The next lemma shows that  \mcbn\ implies closeness to (strict) \mcab\ under the distribution $\D$. The proof is by applying Markov's inequality to Definition \ref{def:m-cal-1}.

\begin{lemma}
    If  $\mS$ is $\alpha$-\mcbd for $\mC, \D$, then for every $c \in \mC$ and $\beta \in [0,1]$ 
    \begin{align}
        \Pr_{\i \sim \D}\left[\abs{\CoVar_{\D_{\i}}[c(\x), \y} \geq \frac{\alpha}{\beta}\right] \leq \beta.
    \end{align}
\end{lemma}

This lemma shows that being $\alpha\beta$-\mcbd\  is closely related to the notion of $(\alpha, \beta)$-\mcab\ in \cite{hkrr2018, gopalan2021multicalibrated}, which roughly says that the $\alpha$-\mcab\ condition holds for all but a $\beta$ fraction of the space $\X$. Conversely, one can show that $(\alpha, \beta)$-\mcab\ gives $(\alpha + \beta \infnorm{\mC})$-\mcbn. We find the single parameter notion of $\alpha$-\mcbn\ more elegant. It is also easy to achieve sample efficiency, since as the next lemma shows, it only requires strong conditional guarantees for large states.

\begin{lemma}
\label{lem:small}
    Let the partition $\mS$ be such that for all $i \in [m]$ where
    \begin{align}
    \label{eq:small}
        \D(S_i) \geq \frac{\alpha}{2m\infnorm{\mC}}
    \end{align}
    it holds that for every $c \in \mC$, 
    \begin{align}
    \label{eq:small-cov}
        \abs{\CoVar_{\D_i}[c(\x), \y]} \leq \frac{\alpha}{2}.
    \end{align}
    Then $\mS$ is $\alpha$-\mcbd for $\mC, \D$.
\end{lemma}

\Mcbn implies that for an {\em average} state in the partition, conditioning on the label does not change the expectation of $c(x)$ much. The proof follows by plugging Equation \eqref{eq:covar-1} in the definition of \mcbn.

\begin{corollary}
\label{cor:bin}
If $\mS$ is $\alpha$-\mcbd for $\mC, \D$, then for $c \in \mC$ and $b \in \zo$,
\begin{align}
    \E_{\i \sim \D}\lt[\Pr_{\D_i}[\y = b]\lt|\E_{\D_i|\y = b}[c(\x)] - \E_{\D_i}[c(\x)]\rt|\rt] \leq \alpha.
\end{align}
\end{corollary}

\paragraph{Extension to the multi-class setting}

In the multi-class setting $\Y =[l]$, so that $l =2$ is exactly the Boolean case considered above. We use $\ind{\y =j}$ to denote the indicator of the event that the label is $j$. The following definition generalizes Definition \ref{def:m-cal-1}:

\begin{definition}
\label{def:m-cal-k}
Let $\D$ be a distribution on $\X \times [l]$ where $l \geq 2$. The partition $\mS$ of $\X$ is $\alpha$-\mcbd for $\mC, \D$ if for every $c \in \mC$ and $j \in [l]$, it holds that 
\begin{align}
\label{eq:m-cal-k}
    \E_{\i \sim \D}\lt[\abs{\CoVar_{\D_i}[c(\x), \ind{\y = j}]}\rt] \leq \alpha. 
\end{align}
\end{definition}

\paragraph{Extension to the bounded real-valued case.}

We now consider the setting where $\Y$ is a bounded interval, by scaling we may consider $\Y = [0,1]$. For  interval $J = [v,w] \subset \Y$, let $\ind{\y \in J}$ be the indicator of the event that $\y \in J$. 

\begin{definition}
\label{def:m-cal-real}
Let $\D$ be a distribution on $\X \times [0,1]$. The partition $\mS$ of $\X$ is $\alpha$-\mcbd for $\mC, \D$ if for every $c \in \mC$ and interval $J \subseteq  [0,1]$, it holds that 
\begin{align}
\label{eq:m-cal-real}
    \E_{\i \sim \D}\lt[\abs{\CoVar_{\D_i}[c(\x), \ind{\y \in J}]}\rt] \leq \alpha. 
\end{align}
\end{definition}

\paragraph{Computational efficiency}

Given these new definitions, a natural question is about the computational complexity of computing multicalibrated partitions. Following \cite{hkrr2018}, this task can be accomplished efficiently given a weak agnostic learner for the class $\mC$. We present a  formal statement of this result for the multi-class setting  in Theorem \ref{thm:alg-multi}. The multi-class setting includes the Boolean labels setting as a special case. For the purposes of omniprediction, we show in Section \ref{sec:real} that the bounded real-valued setting reduces to the multi-class setting.  We defer these results to Section \ref{sec:algo} since we do not consider them to be the main contribution of this work, several of the ideas used in the algorithm and its analysis are present in previous work, they are presented here for completeness.

\eat{
For computing such a partition, we show in Section \ref{sec:real} that it can be reduced to the multi-class case for $l = B/\alpha$.
}

\subsection{Some closure properties of \mcab }
\label{sec:closure}

In this section, we prove that \mcbn\ is closed under two natural operations on the class $\mC$ and the distribution $\D$: 
\begin{enumerate}
\item {\bf Linear combinations of $\mC$: } We take sparse linear combinations of the functions $c \in \mC$.
\item {\bf Conditioning $\D$ on a subset: } We condition the distribution $\D$ on a subset $\X' \subseteq \X$ whose indicator lies in the set $\mC$.  
\end{enumerate}
Again we prove these for the case $\Y = \zo$, but the extension to arbitrary $\Y$ is routine.

\paragraph{Multi-calibration under linear combinations.}

We will typically start with an \mcbd partition for a {\em base} class of bounded or even Boolean functions, such as decision trees or coordinate functions. Since our definition allows the functions $c$ to be real-valued and possibly unbounded, we can consider functions arising from linear combinations over this base class. The motivation for this comes from boosting algorithms like AdaBoost or Logistic Regression, where we take a base class of weak learners, and then construct a strong learner which is a linear combination of the weak learners \cite{freund1997decision, boostingBook}.

We will denote by $\mLC$ the set of all linear functions over $\mC$. 
We associate the vector $w = (w_0, w_1, \ldots)$ with the function $g_w \in \mLC$ defined by $g_w(x) = w_0 + \sum_{j}w_j c_j(x)$ where $c_j \in \mC$ and denote $\|w\|_1 = \sum_{j \geq 1} w_j$ (note we have excluded $w_0$). Let
\begin{equation}\label{eq:mLC}
\mLC(W) = \{ g_w \in \mLC: \|w\|_1 \leq W \} 
\end{equation}
be the set of all {\em $W$-sparse} linear combinations. 
The following simple claim shows that \mcab\ is {\em closed} under taking sparse linear combinations. The parameter $\alpha$ degrades with the sparsity. The proof follows from linearity of covariance, and is given in Appendix \ref{app:mcab}.

\begin{lemma}
\label{lem:linear}
For any $W > 0$, if $\mS$ is $\alpha$-\mcbd for $\mC, \D$, then it is $\alpha W$-\mcbd for $\mLC(W), \D$. 
\end{lemma}

\paragraph{Multi-calibration for sub-populations.}

Let $T \subseteq \X$ be a sub-population such that its indicator function belongs to $\mC$. Let $\D'$ denote the distribution $\D|\x \in T$, where $\D'(x) = \D(x)/\D(T)\; \forall x \in T$. Let $\mS' = \{S_i \cap T\}$ be the partition of $T$ induced by $\mS$.  We will use $\D'_i$ for the distribution $\D'|\x \in S'_i$ (which is the same as $\D|\x \in S'_i$ since $S'_i \subseteq T$), and will denote $p'_i = \E_{\D'_i}[\y]$. Let $\mC' \subseteq \mC$ denote the subset of functions from $\mC$ that are supported on $T$ (functions that are 0 outside of $T$). Note that $\mC'$ is nonempty, since the indicator of $T$ lies in it. The proof of the following result for the sub-population $T$ in Appendix \ref{app:sub-pop}.

\begin{theorem}
\label{thm:sub-pop}
    If $\mS$ is $\alpha$-\mcbd for $\mC, \D$, then $\mS'$ is $\alpha(1 + \infnorm{\mC'})/\D(\X')$-\mcbd for $\mC', \D'$.
\end{theorem}

\section{\Totals for convex loss minimization}
\label{sec:binary}

In  this section we consider the setting of binary labels where $\Y = \zo$.

\subsection{Post-processing for nice loss functions}

Given an $(B, \eps)$-nice loss function, there is a natural post-processing of the canonical predictor $f^\mS$ that we will analyze. Rather than choose the value $k^*(p) \in \R$ which minimizes expected loss under $\B{p}$, we restrict ourselves to the best value from $I_\ell$. This restriction only costs us $\eps$ by the $\eps$-optimality property. 

\begin{definition}
\label{def:k-ell}
Given a nice loss function $\ell$, define the function $k_\ell: [0,1] \rgta I_\ell$ by
\begin{align}
      \label{eq:def-kp}
    k_\ell(p) = \argmin_{t \in I_\ell} \E_{\y \sim \B{p}} \ell(\y,t).
\end{align} 
Given a partition $\mS$ of $\X$,  define the $\ell$-optimized hypothesis $h^\mS_\ell: \X \rgta I_\ell$ as
\[ h^\mS_\ell(x) = k_\ell \circ f^\mS(x). \]
\end{definition}
Since $\ell$ is convex as a function of $t$, so is 
\[ \E_{\y \sim \B{p}} \ell(\y,t) = p\ell(0,t) + (1- p)\ell(1, t).\]
Hence computing $k_\ell$ is a one-dimensional convex minimization problem, a classical problem with several known algorithms \cite{boydBook}. Being able to compute an $\eps'$-approximate solution suffices for us, we can  absorb the $\eps'$ term into the error $\eps$, and pretend that $\ell$ is $(B, \eps + \eps')$-nice instead.

We can view the hypothesis $h^\mS_\ell$ as a function mapping $\mS$ to $I_\ell$, since it is constant on each $S_i \in \mS$, and its range is $I_\ell$. A simple consequence of the definition is that it is the best function in this class for minimizing expected loss.

\begin{corollary}
\label{cor:optimal}
For all functions $h: \mS \rgta I_\ell$, $\ell_\D(h^\mS_\ell) \leq \ell_\D(h)$. 
\end{corollary}
\begin{proof}
    We sample $\i \sim \D$ and then $\x, \y \sim \D_\i$ and show that the inequality holds conditioned on every choice of $\i =i$. Since $\y \sim D_i$ is distributed as $\B{p_i}$,     
    \[ \ell_{\D_i}(h^\mS_\ell) =  \E_{\B{p_i}}[\ell(\y, k_\ell(p_i))]  \leq \E_{\B{p_i}}[\ell(\y, h(S_i))] = \ell_{\D_i}(h) \]
    where the inequality is by Definition \ref{def:k-ell}.
\end{proof}

\subsection{Loss minimization through \Mcab}
\label{sec:main}

Our main result in this section is the following theorem.  
\begin{theorem}
\label{thm:key}
Let $\D$ be a distribution on $\X \times \zo$, $\mC$ be a family of real-valued functions on $\X$ and $\mL(B,\eps)$ be the family of all $(B, \eps)$-nice loss functions.  If the partition $\mS$ is $\alpha$-\mcbd\ for $\mC, \D$, then the canonical predictor $f^\mS$ is an $(\mL, \mC, 2\alpha B + \eps)$-\total.
\end{theorem}

The following lemma is the key ingredient in our result. Informally it says that $c$ has limited distinguishing power within each state of the partition. Specifically, that $c(x)$ is not much better at minimizing a loss than the function obtained by taking its conditional expectation within each state of the partition $\mS$. 
\begin{lemma}
\label{lem:better}
Given $\ell \in \mL$ and $c \in \mC$, define the predictor $\hat{c}: \mS \rgta I_\ell$ by
\[ \hat{c}(x)  = \clip\lt(\E_{\D_i}[c(x)], I_\ell\rt) \ \text{for} \ x \in S_i. \]
We have
\begin{align}
     \ell_\D(\hat{c}) \leq \ell_\D(c) + 2\alpha B + \eps.
\end{align}
\end{lemma}
\begin{proof}
    By the convexity of $\ell$ we have
    \begin{align}
    \label{eq:use-conv}
    \ell_\D(c) = \E_{\D}[\ell(\y, c(\x))] = \E_{\i \sim \D}\E_{\y \sim \D_\i}\E_{\x \sim \D_\i|\y}[\ell(\y, c(\x)] \geq \E_{\i \sim \D}\E_{\y \sim \D_\i}\lt[\ell\lt(\y, \E_{\x \sim \D_\i|\y}[c(\x)]\rt)\rt].
    \end{align}
    By Lemma \ref{lem:tech}, 
    \[ 
    \ell\lt(\y, \E_{\x \sim \D_i|\y}[c(\x)]\rt) \geq \ell\lt(\y, \clip\lt(\E_{\x \sim \D_i|\y}[c(\x)], I_{\ell}\rt)\rt) - \eps.\] 
    Plugging this into Equation \eqref{eq:use-conv}, we get
    \begin{align}
    \label{eq:loss-1}
        \ell_\D(c) \geq \E_{\i \sim \D}\E_{\y \sim \D_\i}\lt[\ell\lt(\y, \clip\lt(\E_{\x \sim \D_i|\y}[c(\x)], I_{\ell}\rt)\rt)\rt] - \eps.
    \end{align}
    From the definition of $\hat{c}$,
    \begin{align}
    \label{eq:loss-2}
    \ell_\D(\hat{c}) = \E_{\D}[\ell(\y, \hat{c}(\x))] = \E_{\i \sim \D}\E_{\y \sim \D_\i}\lt[\ell\lt(\y, \clip\lt(\E_{\x \sim \D_i}[c(\x)], I_{\ell}\rt)\rt)\rt].     
    \end{align}
    Subtracting Equation \eqref{eq:loss-1} from Equation \eqref{eq:loss-2} we get
    \begin{align}
        \ell_\D(\hat{c})  - \ell_\D(c) \leq &
        \E_{\i \sim \D}\E_{\y \sim \D_\i}\lt[\ell\lt(\y, \clip\lt(\E_{\x \sim \D_i}[c(\x)], I_{\ell}\rt)\rt) - \ell\lt(\y, \clip\lt(\E_{\x \sim \D_i|\y}[c(\x)], I_{\ell}\rt)\rt) \rt] + \eps .   \label{eq:loss-3}
    \end{align}
    Since $\ell$ is $B$-Lipschitz on $I_{\ell}$  and $\clip(t,I_{\ell})$ is $1$-Lipschitz as a function of $t$, 
    \begin{align}
    \ell\lt(\y, \clip\lt(\E_{\x \sim \D_i}[c(\x)], I_{\ell}\rt)\rt) &  - \ell\lt(\y, \clip\lt(\E_{\x \sim \D_i|\y}[c(\x)], I_{\ell}\rt)\rt) \notag\\
    &\leq B\abs{\clip\lt(\E_{\x \sim \D_i}[c(\x)], I_{\ell}\rt) - \clip\lt(\E_{\x \sim \D_i|\y}[c(\x)], I_{\ell}\rt)}\notag\\
    &\leq B\abs{ \E_{\x \sim \D_i}[c(\x)] - \E_{\x \sim \D_i|\y}[c(\x)]}.\label{eq:lipschitz+}
    \end{align}
    Plugging this into Equation \eqref{eq:loss-3} gives
    \begin{align}
    \ell_\D(\hat{c})  - \ell_\D(c) - \eps &\leq
        B\E_{\i \sim \D}\E_{\y \sim \D_\i} \lt| \E_{\x \sim \D_i}[c(\x)] - \E_{\x \sim \D_i|\y}[c(\x)]\rt|\label{eq:use-this}\\
        & = B\E_{\i \sim \D}\lt[\sum_{b \in \zo}\Pr_{\D_\i}[\y =b]\lt|\E_{\x \sim \D_i}[c(\x)] - \E_{\x \sim \D_i|\y}[c(\x)]\rt|\rt]\notag\\
        &= B\sum_{b \in \zo}\E_{\i \sim \D}\lt[\Pr_{\D_\i}[\y =b]\lt|\E_{\x \sim \D_i}[c(\x)] - \E_{\x \sim \D_i|\y}[c(\x)]\rt|\rt]\notag\\
        & \leq B \sum_{b \in \zo}\alpha = 2\alpha B\notag,
    \end{align}
    where the last inequality follows by the multicalibration condition (Equation \eqref{eq:multicalibration_corollary}).
\end{proof}

As a consequence we can now prove Theorem \ref{thm:key}. 
\begin{proof}[Proof of Theorem \ref{thm:key}]
Let $h^\mS_\ell= k_\ell \circ f^\mS$ be the $\ell$-optimized hypothesis. It suffices to show that for any $c \in \mC$
    \begin{align}
    \label{eq:key}
        \ell_\D(h^\mS_\ell) \leq \ell_\D(c) + 2\alpha B + \eps. 
    \end{align}     
    For any $c \in \mC$, we have 
    \[ \ell_\D(h^\mS_\ell) \leq \ell_\D(\hat{c}) \leq \ell_\D(c) + 2\alpha B +\eps \]
    where the first inequality is by Corollary \ref{cor:optimal}, which applies since $\hat{c}$ is a function mapping $\mS$ to $I_\ell$. The second is by Lemma \ref{lem:better}.
\end{proof}

Consider the family $\mLC(W)$ of linear combinations over $\mC$ of weight at most $W$. By Lemma \ref{lem:linear}, $\mS$ is $\alpha W$-\mcbd\ for $\mLC(W)$. Applying Theorem \ref{thm:key},  we derive the following corollary.

\begin{corollary}
\label{cor:lipschitz}
    Let $\D$ be a distribution on $\X \times \zo$, $\mLC(W)$ be linear functions in $\mC$ of with $\onenorm{w} \leq W$ (see eq.~\ref{eq:mLC}) and $\mL = \mL(B,\eps)$ be the family of all $(B, \eps)$-nice loss functions.  If the partition $\mS$ is $\alpha$-\mcbd\ for $\mC, \D$, then $f^\mS$ is an $(\mL, \mLC(W), 2\alpha BW + \eps)$-\total.
\end{corollary}

To interpret this, assume we have an $(B, \eps)$-nice loss function and we wish to have a predictor that is within $2\eps$ of any function in $\mLC(W)$. Corollary \ref{cor:lipschitz} says that it suffices to have an $\alpha$-\mcbd\ partition where $\alpha = \eps/2BW$. Note that algorithms for computing such partitions have running time which is polynomial in $1/\alpha$, which translates to running time polynomial in $BW/\eps$. 

We derive a corollary for sub-populations follows from Theorem \ref{thm:key} and \ref{thm:sub-pop}. For two families of functions $\mathcal{T}, \mathcal{P}: \X \rgta \R$, we define their product as 
\[ \mathcal{T} \times \mathcal{P} = \{c: c(x) = T(x)P(x), T \in \mathcal{T}, P \in \mathcal{P} \}. \]
Note that in the case when $T \in \mathcal{T}$ is binary-valued, $\mathcal{T} \times \mathcal{P}$ contains the restriction of every $P \in \mathcal{P}$ to the support of $T$.

\begin{corollary}
\label{cor:sub_pop}
    Let $\D$ be a distribution on $\X \times \zo$, $\mathcal{T}$ be a family of binary-valued functions on $\X$,
    $\mathcal{P}$ be a family of real-valued functions and $\mL(B,\eps)$ be the family of all $(B, \eps)$-nice loss functions. Let the partition $\mS$ be $\alpha$-\mcbd\ for $\mathcal{T} \times \mathcal{P}, \D$. For $T \in \mathcal{T}$,  let $\alpha'=\alpha(1 + \infnorm{\mathcal{P}})/\D(T)$. Then the canonical predictor $f^\mS$ is an $(\mL, \mathcal{P}, 2\alpha' B + \eps)$-\total for the sub-population $T$.
\end{corollary}

To informally instantiate this for a simple case, let $\mathcal{T}, \mathcal{P}$ be the class of decision trees of depth $d_1$ and $d_2$ respectively. Let the partition $\mS$ be multicalibrated with respect to decision trees of depth $d_1+d_2$. If we now consider any sub-population $T$ identified by decision trees of depth $d_1$, then the above result implies that the canonical predictor $f^\mS$  is an \total for $T$, when compared against the class of decision trees of depth $d_2$ evaluated on $T$.

\subsection{Limits for \totals from \mcab}
\label{sec:limits}

Corollary \ref{cor:lipschitz} shows  that \mcab for $\mC$ gives \totals for $\mLC$. It is natural to ask whether we can get \totals for a richer class of functions using \mcab for $\mC$. A natural candidate would be thresholds of functions in $\mC$:
\[ \Hspc = \{ \ind{c(x) \geq v}:  c \in  \mC, v \in \R\}.\]
Another natural extension would be to relax the convexity condition for loss functions in $\mL$. We present a simple counterexample which shows that \mcab for $\mC$ is insufficient to give \totals for both these classes. This shows that a significant strengthening of the bound from Corollary \ref{cor:lipschitz} might not be possible.  

\begin{lemma}
There exists a distribution $\D$, a set $\mC:\X \rgta \R$ of functions, and a $0$-\smcbd partition $\mS$ for $\mC, \D$ such that for any $\delta < 1/4$, 
\begin{itemize}
\item $f^\mS$ is {\em not} an $(\mL, \Hspc, \delta)$-\total for any $\mL$ containing the $\ell_1$ loss function.
\item  $f^\mS$ is {\em not} an $(\mL, \mC, \delta)$-\total for any $\mL$ containing the (non-convex) loss function $\ell(y, t) = |y - \ind{t \geq 0}|$.
\end{itemize}
\end{lemma}
\begin{proof}
Let $\D$ be  the distribution on $\zo^3 \times \zo$ where $\x \sim \zo^3$ is sampled uniformly and $\y = \ind{\sum_{i=1}^3 \x_i \equiv 0 \bmod 2}$ is the negated Parity function. Let $\mC = \{\sum_i w_i x_i - w_0 \}$ be all affine functions. We claim the trivial partition $\mS = \{\zo^3\}$ is $0$-\smcbd for $\mC, \D$.   This is because every $x_i$ is independent of $\y$, so their covariance is $0$. By linearity of expectation, the same is true for all functions in $\mC$. Thus $f^\mS(x) = 1/2$ for every $x \in \zo^3$.  

Now consider the $\ell_1$ loss. A simple calculation shows that for every $k:\zo \rgta \R$, $\bll_1(k \circ f^\mS, \D) \geq 1/2$. 
In contrast $h(x) = \ind{x_1 + x_2 + x_3 \geq  1.5} \in \Hspc$ gives  $\bll_1(h, \D) =1/4$, since it gets the two middle layers correct. This proves part (1).

To deduce part (2), let $\ell(y, t) = |y - \ind{t \geq 0}|$ and $g(x) = x_1 + x_2 + x_3 -1.5 \in \mC$. Note that 
\[ 1/4 = \bll_1(h, \D) = \E_{\D}[|h(\x) - \y|] = \E_{\D}[|\ind{g(\x) \geq 0} - \y|] = \E_{\D} [\ell(\y, g(\x))] = \bll(g, \D).  \]
In contrast, for any $k: [0,1] \rgta \R$, it follows that $\bll(k \circ f^\mS, \D) \geq 1/2$.
\end{proof}

In part (2), the loss function $|y - \ind{t \geq 0}|$ is not Lipshcitz or differentiable in $t$. We can ensure both these conditions by replacing it with the sigmoid function, which still preserves the correlation with parity, at the cost of some reduction in $\delta$ for which the bound holds \cite{Kalai04}. 

\eat{
\subsection{Loss minimization through Multi-calibration}

In this section, we will prove the following result.

\begin{theorem}
\label{thm:key}
Let $\ell$ be a $(I,B, \eps)$-nice loss function $\ell$.  Let the partition $\mS$ be $\alpha$-\mcbd\ for $\mC, \D$, and let $h^\mS_\ell$ be the $\ell$-optimized hypothesis. Then for any $c \in \mC$  and $i \in [m]$,
    \begin{align}
    \label{eq:key}
        E_{\D}[\ell(\y, h^\mS_\ell(\x)] \leq \E_{\D_i}[\ell(\y, c(\x))] + 2\alpha B + \eps. 
    \end{align}     
\end{theorem}
\begin{proof}
    We will show that
    \begin{align}
    \label{eq:key1}
        E_{\D_i}[\ell(\y, k_\ell(p_i)] \leq \E_{\D_i}[\ell(\y, c(\x))] + 2\CoVar_{\D_i}[c(\x), \y] B + \eps. 
    \end{align}     
    
    To simplify notation, we denote $|\CoVar_{\D_i}[c(\x), \y)]|$ by $\alpha_i$. The $\alpha$-\mcbn\ condition guarantees that
    \[ \sum_i \D(S_i)\alpha_i \leq \alpha. \]
    We start with states where $p _i(1 -p_i) > 0$. Assume without loss of generality that $p_i \leq 1/2 \leq 1- p_i$. 
    Recall that $\ell$ is $(I, B, \eps)$-nice. Fix $i \in [m]$. Define $z_0, z_1, v_0 , v_1$ by:
\begin{align*}
    z_0 \defeq \E_{(\x, \y) \sim \D_i} \left[ c(\x) \mid \y = 0\right], &   \ \ z_1 \defeq \E_{(\x, \y) \sim \D_i} \left[ c(\x) \mid \y = 1\right],\\
     v_0 = \clip(z_0, I), & \ \ v_1 = \clip(z_1, I).
\end{align*} 
    By the definition of \mcab\ (Equation \eqref{eq:m-cal-01}) we have
\begin{align*}
    |z_1 - z_0| &= \left|\E_{(\x, \y) \sim \D_i} \left[  c(\x) \mid \y = 1\right] - \E_{(\x, \y) \sim \D_i} \left[ c(\x) \mid \y = 0\right]\right| \leq \frac{\alpha_i}{p_i(1 - p_i)}.
\end{align*}
Since $\ell(0, t)$ is $B$-Lipshcitz on $I$, and  $\clip(t, I)$ is $1$-Lipschitz as a function of $t$, 
\begin{align}
\label{eq:not-key}
    |\ell(0,v_1) - \ell(0, v_0)| \leq B|v_1 - v_0| \leq B|z_1 - z_0| \leq \frac{\alpha_i B}{p_i(1 - p_i)}. 
\end{align}

We now claim the inequalities  
\begin{align}
\label{eq:chain1}
\ell(1, v_1)  \leq \ell(1, z_1) + \eps \leq \E_{\D_i} \left[ \ell(1, c(\x))\mid \y = 1\right] + \eps\\
\label{eq:chain2}
\ell(0, v_0)  \leq \ell(0, z_0) + \eps \leq \E_{\D_i} \left[ \ell(0, c(\x))\mid \y = 0\right] + \eps
\end{align}
The first inequality is because clipping only increases the loss by $\eps$, and the second is by Jensen's inequality and the definition of $z_0/z_1$. We now have
\begin{align*}
\E_{(\x, \y) \sim \D_i}[\ell(\y, k_\ell(p_i))]  &\leq  \E_{(\x, \y) \sim \D_i}[\ell(\y, v_0)] \ \ \ \ \text{by Definition of } k_\ell \\
& = p_i\ell(1,v_0) + (1- p_i)\ell(0, v_0) \\
& \leq p_i\left(\ell(1, v_1) + \frac{\alpha_i B}{p_i(1 - p_i)} \right)+   (1 - p_i)\ell(0, v_0)  \ \ \ \ \text{By Equation} \ \eqref{eq:not-key}\\
& \leq p_i\left(\E_{\D_i} \left[ \ell(1, c(\x))\mid \y = 1\right] + \eps\right) + (1 - p_i)\left(\E_{\D_i} \left[ \ell(0, c(\x))\mid \y = 0\right] + \eps\right)\\
& + \frac{\alpha_i B}{(1 - p_i)} \ \ \ \ \ \ \ \text{By Equations} \ \eqref{eq:chain1}, \eqref{eq:chain2}\\
& \leq \E_{\D_i}[\ell(\y, c(\x)] + 2\alpha_i B + \eps. \ \ \ \ \ \ \ \text{Since } 1 - p_i \geq 1/2
\end{align*}
which proves Equation \eqref{eq:key1}.

In the case where $p_i(1 -p_i) = 0$, assume $p_i =1$ for $S_i$, the case $p_i =0$ is symmetric. Then using Equation \eqref{eq:chain1} we have
\[ \ell(1, k_\ell(1)) \leq \ell(1, v_1) \leq \ell(1, z_1) + \eps \leq \E_{\D_i}[\ell(1, c(\x)] + \eps \]
where the last inequality uses the fact that $\D_i|\y =1$ is the same as $\D_i$.
 Note that here $\alpha_i = \CoVar_{\D_i}[c(\x), \y] = 0$ since $\y$ is constant, so this is Equivalent  to Equation \eqref{eq:key1}.
 
We now average Equation \eqref{eq:key1} over the various states according to $\D(S_i)$ to complete the proof:
\begin{align*}
    \E_{(\x, \y) \sim \D}[\ell(\y, f^\ell(\x))] &= \sum_{i \in [m]} \D(S_i)\E_{\D_i}[\ell(\y, t_i)] \leq \sum_{i \in [m]}\D(S_i)\lt(\E_{\D_i}[\ell(\y, c(\x))] + 2\alpha_i B + \eps\rt)\\
    & = \E_{\D}[\ell(\y, c(\x))] + 2\alpha B + \eps.
\end{align*}
\end{proof}
}

\eat{
\begin{lemma}
\label{lem:lipschitz}
Let $\ell$ be an $(I, B, \eps)$-reasonable loss.
For any $i \in [m]$ and predictor of the form $g_w(x) = w_0 + \sum_j w_j c_j(x)$,
\begin{align*}
    \E_{(\x ,\y) \sim \D_i} [\ell(\y, t_i)]  \leq  \E_{(\x, \y) \sim \D_i} \left[ \ell(\y, g_w(\x))\right] + \alpha B \|w\|_1 + \eps,
\end{align*}
where $t_i = \arg\min_{t \in I} \E_{(\x ,\y) \sim \D_i} [\ell(\y, t)].$
\end{lemma}
\begin{proof}
For the proof we fix $i$ so we can drop $i$ subscripts and write $\D=D_i$, $\D^0=\D^0_i$. Define $z_\pm$ by:
\begin{align*}
    z_+ &\defeq \E_{(\x, \y) \sim \D} \left[ g_w(\x) \mid \y = 1\right]\\
    z_- &\defeq \E_{(\x, \y) \sim \D} \left[ g_w(\x) \mid \y = -1\right].
\end{align*}
We have
\begin{align*}
    \frac{z_+-z_-}{2} &= \E_{(\x, \y) \sim \D^0} \left[ \y g_w(\x) \mid \y = 1\right] - \E_{(\x, \y) \sim \D^0} \left[ \y g_w(\x) \mid \y = -1\right]\\
    &= \E_{(\x, \y) \sim \D^0} \left[ \y g_w(\x) \right]\\
    &=  w_0 \E_{(\x, \y) \sim \D^0}[\y] + \sum_{j \geq 1}w_j\E_{(\x, \y) \sim \D^0} \left[ \y g_w(\x) \right]\\
    \frac{|z_+-z_-|}{2} &\leq \alpha\sum_{j\geq 1}|w_j| \leq  \alpha \|w\|_1. 
\end{align*}
In the above we have used the definition of multi-calibration, the fact that $\E_{\D^0}[\y]=0$ and the triangle inequality. 
Thus $|z_+-z_-|\leq \alpha \|w\|_1$. Let $t_\pm$ be $z_\pm$ clipped to $I$, respectively. It is not difficult to see that clipping can only increase loss by $\eps$, e.g.,
\begin{align}
\ell(-1, t_-)  \leq \ell(-1, z_-) + \eps \leq \E_{\D} \left[ \ell(-1, g_w(\x))\mid \y=-1\right] + \eps
\end{align}
The first inequality is because clipping only increase the loss by $\eps$, and the second is by Jensen's inequality and the definition of $z_-$. A similar inequality can be shown for $\ell(1, t_+)$. 

It is also not hard to see that $|t_+-t_-|\leq \alpha \|w\|_1$ since clipping can only reduce distances. Thus,
\begin{align*}
\E_{(\x, \y) \sim \D}[\ell(\y,t)]  \leq & \E_{(\x, \y) \sim \D}[\ell(\y,t_-)]\\
=& \Pr_\D[\y =-1]\ell(-1,t_-) + \Pr_\D[\y =1]\ell(1, t_-) \\
\leq& \Pr_\D[\y =-1]\ell(-1,t_-) + \Pr_\D[\y =1]\ell(1, t_+) + \alpha \|w\|_1 B \\
\leq& \Pr[\y=-1]\E_{\D} \left[ \ell(-1, g_w(\x))\mid \y=-1\right] + \eps 
\\
& + \Pr[\y=1]\E_{\D} \left[ \ell(1, g_w(\x))\mid \y=1\right] + \alpha \|w\|_1 B\\
=& \E_{(\x, \y) \sim \D} \left[ \ell(\y, g_w(\x))\right] + \eps +  \alpha \|w\|_1 B.
\end{align*}
\end{proof}
}

\eat{
\subsection{Multi-dimensional Lipschitz cost}

\newcommand{\cost}{{l}}

\paragraph{Assumption.}
Suppose there is a convex decision set $T \subseteq \R^d$ and a cost function $\cost:\zo \times T \rightarrow \R$ where both $\cost(0, t)$ and $\cost(1, t)$ are convex and $B$-Lipschitz it $t$. 

As before, we define the extension $\cost(p, t) = \E_{\y \sim \B{p}}[\cost(\y, t)$.  For a partition $\mS$, we define 
\begin{align*} 
    t_i &= \arg\min_{t \in T} \cost(p, t)\\
    f^\mS_\cost(x) &= t_i \ \forall x \in S_i 
\end{align*}

With this setup we have:

\begin{lemma}
     Let $\mS$ be $\alpha$-\mcbd\ for $\mC, \D$. Then for any $c \in \mC$ we have
    \[ \E_{\mD}[\ell(\y, f^\mS_\cost(\x)] \leq \ell(\y, c(\x)) + \alpha B + \eps. \]
\end{lemma}

\begin{lemma}
\label{lem:lipschitz}
For any $i \in [m]$ and linear decision function $g^w:\X\rightarrow T$ of the form $g^w(x) = \langle w_{0,k} + \sum_j w_{j,k}(x)\rangle_{k=1}^d$,
\begin{align*}
    \min_{t \in T} \E_{(\x ,\y) \sim \D_i} [\ell(\y, t)]  \leq  \E_{(\x, \y) \sim \D_i} \left[ \ell(\y, g^w(\x))\right] +\alpha B \|w\|_1.
\end{align*}
\end{lemma}
Note that in the above we have assumed that the linear function is bounded to $T$ as well.
\begin{proof}
For the proof we fix $i$ so we can drop $i$ subscripts and write $\D=D_i$, $\D^0=\D^0_i$. Define decision vectors $t^\pm \in T$ by:
\begin{align*}
    t^+ &\defeq \E_{(\x, \y) \sim \D} \left[ g^w(\x) \mid \y = 1\right]\\
    t^- &\defeq \E_{(\x, \y) \sim \D} \left[ g^w(\x) \mid \y = -1\right].
\end{align*}
For any $k \in [d],$ we have:
\begin{align*}
    \frac{t^+_k-t^-_k}{2} &= \E_{(\x, \y) \sim \D^0} \left[ \y g^w_k(\x) \mid \y = 1\right] - \E_{(\x, \y) \sim \D^0} \left[ \y g^w_k(\x) \mid \y = -1\right]\\
    &= \E_{(\x, \y) \sim \D^0} \left[ \y g^w_k(\x) \right]\\
    &=  w_{0,k} \E_{(\x, \y) \sim \D^0}[\y] + \sum_{j \geq 1}w_{j,k}\E_{(\x, \y) \sim \D^0} \left[ \y g^w_k(\x) \right]\\
    \frac{\|t^+-t^-\|}{2} &\leq \frac{\|t^+-t^-\|_1}{2} \leq\alpha\sum_{j\geq 1, k \in [d]}|w_{j,k}| \leq \alpha \|w\|_1
\end{align*}
In the above we have used the definition of multi-calibration, the fact that $\E_{\D^0}[\y]=0$ and the triangle inequality. 
Thus $\|t^+-t^-\|\leq \|t^+-t^-\|_1 \leq \alpha \|w\|_1$. By Jensen's inequality,
\begin{align}
\cost(-1, t^-)  \leq  \E_{\D} \left[ \cost(-1, g^w(\x))\mid \y=-1\right],
\end{align}
and similarly for $\cost(1, t^+)$. Then, since $\|t^--t^+|\leq 2\alpha \|w\|_1$ and $c(1, \cdot)$ is $B$-Lipschitz, $c(1, t^-)\leq c(1, t^+) + 2\alpha \|w\|_1 B$,
\begin{align*}
\min_{t \in T}\E_{(\x, \y) \sim \D}[\cost(\y,t)]  \leq & \E_{(\x, \y) \sim \D}[\cost(\y,t^-)]\\
=& \Pr_\D[\y =-1]\cost(-1,t^-) + \Pr_\D[\y =1]\cost(1, t^-) \\
\leq& \Pr_\D[\y =-1]\cost(-1,t^-) + \Pr_\D[\y =1]\left(\cost(1, t^+) + 2\alpha \|w\|_1 B\right) \\
\leq& \Pr[\y=-1]\E_{\D} \left[ \cost(-1, g^w(\x))\mid \y=-1\right] + 
\\
& \Pr[\y=1]\E_{\D} \left[ \cost(1, g^w(\x))\mid \y=1\right] + 2 \alpha \|w\|_1 B Pr[\y=1]\\
=& \E_{(\x, \y) \sim \D} \left[ \cost(\y, g^w(\x))\right] +  Pr[\y=1] \cdot 2 \alpha \|w\|_1 B .
\end{align*}
If $\Pr[\y=1]\leq 1/2$, then the above is the lemma statement. Otherwise, an analogous chain of inequalities using $t^+$ instead of $t^-$  shows gives the same bound above with $\Pr[\y=-1]$, which completes the proof.
\end{proof}
}

\eat{
For $p \in [0,1]$, let $\y \sim \B{p}$ denote the Bernoulli random variable with parameter $p$. We can extend $\ell$ to  a function $\ell: [0, 1] \times \R \rightarrow \R^+$ by defining
\begin{align}
\label{eq:def-lpt}
    \ell(p, t) = \E_{\y \sim \B{p}}[\ell(\y,t)] = p\ell(1, t) +  (1 - p)\ell(0, t).
\end{align} 
to be the expected loss when we predict $t$ and the true label is $\y \sim \B{p}$. It follows that $\ell(p, t)$ is a non-negative  convex function of $t$ for every $p$.

We consider predictors obtained from a \mcbd\ partition by a simple post-processing step. To motivate these predictors, imagine first that we know the ground-truth values $f^*(x) = \E_\D[\y|\x = x]$ for every $x \in \X$. How would we make predictions to minimize our expected loss? 

To answer this question we need some notation. 
Define the function $k^*: [0,1] \rgta \R$ by
    \[ k^*(p) = \arg\min_{t \in \R} \ell(p,t). \]
If $\ell(p, t)$ does not have a unique minimum, then there might be multiple choices for $k^*(p)$, and we can break ties arbitrarily. Note that $k^*$ depends on $\ell$ but is independent of the distribution $\D$ and the dataset. For instance, for the $\ell_1$ loss, $\ell_1(y, t) = |y -t|$, we have $k^*(p) = \sign(p -1/2)$, whereas for the $\ell_2$ loss  $\ell_2(y,t) = (y -t)^2$, we have $k^*(p) = p$. 

Returning to the problem of loss minimization, for a point $x$, we suffer expected loss $\ell(f^*(x), t)$. Hence, by the definition of $k^*$,
a prediction strategy which minimizes the expected loss over all predictors is to predict $k^*(f^*(x))$ for every $x \in \X$. The trouble is that this require knowledge of $f^*$. 

Now suppose $\mS$ is a \mcbd\ partition for $\mC, \D$. For $x \in S_i$, the value $p_i = \E_{\D_i}[\y]$ represents our best guess for $f^*(x)$. Hence it is natural to predict $k^*(p_i)$ at points in $S_i$ to minimize the loss. 
In what follows, we will prove that this predictor compares favorably to the best linear predictors of bounded $\ell_1$ norm.

\subsection{Lipschitz loss functions and Predictors}
}

\section{Agnostic boosting via \mcab}
\label{sec:ag}

In this section we apply our results to the setting of agnostic learning, where the labels are Boolean, and the relevant loss is the classification error aka $0$-$1$ loss. 

\begin{definition}
Let $\mH : \X \rgta \zo$ be a family of classifiers. Given a distribution $\D$ on $\X \times \zo$, for $h \in \mH$ let its classification error be $\err(h) = \Pr_{\sD}[h(\x) \neq \y]$ and $\opt(\mH) = \min_{h \in \mH} \err(h)$.  An algorithm is an agnostic learner for a class $\mH$ if it produces a hypothesis $f$ such that $\err(f)  \leq \opt(\mH) + \eps$. 
\end{definition}

The $\ell_1$ loss defined by $\ell_1(y,t) = |y -t|$ becomes the classification error for Boolean functions, since for any $h:\X \rgta \zo$,
\[    \err(h) = \Pr_{\sD}[y \neq h(\x)] = \E_{\sD}[|\y - h(\x)|] =\E_{\sD}[\ell_1(y, h(\x))]. \]

\begin{definition}
A class $\mH$ is $(\eps, W)$-approximated by the class $\mC$ if for every $h \in \mH$ and $\eps > 0$, there exists $g_w \in \mLC(W)$ such that 
\[\E_{\sD} \left[|g_w(x) - h(\x)|\right] \leq \eps .\]
\end{definition}

For $p \in [0,1]$ the function
\[ \\E_{\y \sim \B{p}} \ell_1(\y, t) = p|1 - t| + (1- p)|t| \]
is minimized in $t$ by the function $k_{\ell_1}(p) = \ind{p \geq 1/2}$. 
Hence the $\ell_1$-optimized hypothesis $h^\mS_{\ell_1}$ outputs $1$ for $x \in S_i$ where $p_i \geq 1/2$ and $0$ elsewhere.

We are now ready to state our main result about agnostic learning via multi-calibration. We show that \mcab\  gives a predictor  that is competitive with the best predictor in a possibly much more powerful class $\mH$. Since \mcab\ can be achieved using a weak learner for $\mC$, this implies a form of agnostic boosting.


\begin{theorem}
\label{thm:ag}
Let $\mH$ be a class that is $(\eps/2, W)$-approximated by $\mC$. If the partition $\mS$  is $\eps/2W$-\mcbd\ for $\mC, \D$, then $\err(h^\mS_{\ell_1}) \leq \opt(\mH) + \eps$. Hence an algorithm that outputs such a partition gives an agnostic learner for $\mH$. 
\end{theorem}
\begin{proof}
Let $h =\arg\min_{h \in \mH}\err(h)$ so that $\err(h) = \opt(\mH)$. 
There exists $g_w \in \mLC(W)$ such that 
\[\E_{\sD} [|g_w(x) - h(\x)|] \leq \eps/2 .\]
Hence by the triangle inequality,
\begin{align*} 
    \E_{\sD}[\ell_1(\y, g_w(\x)] &\leq \E_{\sD}[|\y - h(\x)|] + \E_{\sD}[|g_w(\x) - h(\x)|] \leq \err(h) + \eps/2.
\end{align*}
The $\ell_1$ loss function is $1$-Lipschitz, and $(\R, 1, 0)$-nice. By Corollary \ref{cor:lipschitz} we conclude that since $\mS$ is $\eps/2W$ \mcbd for $\mC, \D$, 
\[ 
   \E_{(\x ,\y) \sim \D} [\ell(\y, h^\mS_{\ell_1}(\x))]  \leq  \E_{(\x, \y) \sim \D} \left[ \ell(\y, g_w(\x))\right] + \eps/2 \leq \err(h) + \eps. \]
But since $h^\mS_{\ell_1}$ is a Boolean function, we have
\[ \E_{(\x ,\y) \sim \D} [\ell(\y, h^\mS_{\ell_1}(\x))]  = \err(h^\mS_{\ell_1}). \]
Hence $\err(h^\mS_{\ell_1}) \leq \err(h) + \eps$, which proves our claim.
\end{proof}

Some comments on the connection of Theorem \ref{thm:ag} to other work:
\begin{itemize}

\item The agnostic learnability of a class $\mH$ that is $\ell_1$ approximated by $\mC$ was first proved by \cite{KKMS08} using linear programming; see  \cite{feldman2009distribution, kk09} for a boosting based approach. Our result reproves this through \mcab and the algorithm of \cite{MansourM2002}. 

\item Agnostic boosting was introduced in the work of  \cite{SBD2}. The agnostic boosting abilities of the \cite{MansourM2002} algorithm are analyzed in \cite{KalaiMV08}, however they bound the error by $\opt(\mC)$. Our work shows that this algorithm is even more powerful; one can compare it to $\opt(\mH)$. 

\item Theorem \ref{thm:ag} gives an upper bound, but it is not a tight characterization of the error of the hypothesis $h^\mS_{\ell_1}$.  Indeed, in subsection \ref{sec:better}, we give an example showing that the error of the hypothesis can be much smaller than $\opt(\mLC)$. 

\item  For the purposes of agnostic boosting, it is often desirable to keep the same marginal distribution over $\X$. Following Kalai-Kanade \cite{kk09} and Feldman \cite{feldman2009distribution}, one can reduce \mcab\ to weak agnostic learning over a distribution with the same marginal on $\X$, by adding noise to the labels. 
\end{itemize}

\subsection{\Mcab can be better than $\opt$}
\label{sec:better}

In the literature of agnostic boosting (specifically \cite{kk09}), a  \textit{$\gamma$-weak agnostic learner} \cite{kk09} for ${\cal H}$ is defined as an algorithm that outputs $c \in {\cal C}$ such that, 
\begin{align}
    \label{eq:kk09}
\cor_\D(c) \geq \gamma \sup_{h \in {\cal H}} \cor_\D(h),
\end{align}
for $\gamma > 0$ in terms of \textit{correlation} $\cor_\D(h) \in [-1,1]$ defined as, \[\cor_\D(h):=\Pr_{\D}[h(x) = y] - \Pr_{\D}[h(x) \neq y].\]
It is shown that a $\gamma$ weak agnostic learner can be used to achieve error within $\eps$ of $\opt(\mH)$ using time and samples $\poly(1/\gamma, 1/\eps)$.
\cite{kk09} do not explicitly limit what class $\mH$ can be relative to $\mC$, beyond saying that it should satisfy condition Equation \eqref{eq:kk09}. 
However it is not difficult to see that this definition limits agnostic boosting to ${\cal H} \subseteq  \mLC$, i.e., linear combinations  as defined in Equation \eqref{eq:mLC}. Therefore, agnostic boosting (as in \cite{KalaiMV08, kk09}) essentially learns linear combinations of ${\cal C}$, while we have shown that multi-calibrated predictors also satisfy this goal. 
Indeed, our proof shows that the marginal distribution on $\X$ can be considered fixed, we just require Equation \eqref{eq:kk09} to hold for all $h \in \mH$ and marginal distributions on the label $\y$. 

\begin{lemma}
\label{lem:adam}
    For Equation \eqref{eq:kk09} to hold for every $h \in \mH$, we must have ${\cal H} \subseteq  \mLC$.
\end{lemma}
\begin{proof}
For the purpose of contradiction consider any $h \notin \mLC$.  Let $g: \X \rightarrow \R$ be the projection of $h$ onto $\mLC$, i.e., that minimizes $\E_\D[(g(x)-h(x))^2]$. Since $\X$ and ${\cal C}$ are finite, $g$ is bounded. Thus consider \[f^*(x) := \frac{1}{2} + \frac{h(x)-g(x)}{\max_{\X} |h(x)-g(x)|} \in [0,1].\] In particular, for the distribution $\D'$ which has the same marginal over $\X$ but so that $f^*(x)=\E_{\D'}\y|x]$, it is not difficult to see that $\cor_{\D'}(c)=0$ for all $c \in {\cal C}$ and yet $\cor_{\D'}(h)>0$ violating the definition of a weak learner. 
\end{proof}

We now show that being \mcbd is a \textit{stronger} notion than $\opt(\mLC)$ in that being multi-calibrated with respect to ${\cal C}$ implies a loss that may be even lower than the best classifier in $\mLC$. This says that the upper bound on the classification error in terms of $\opt(\mLC)$ given by Theorem \ref{thm:ag} may not always be tight.  

\begin{theorem}
    For any $\eps > 0$, there exists a distribution $\D$ on $\zo^2 \times \zo$ and a set of functions $\mC:\X \rgta \zo$ such that for any \mcbd partition $\mS$, we have $\opt(\mLC) - \err(h^\mS_{\ell_1}, \D) \geq 1/4 - \eps/2$.
\end{theorem}
\begin{proof}
Let $\X =\{0,1\}^2, \Y = \zo$, the marginal distribution over $x \in \X$ uniform, and the conditional distribution over $\Y$ given by
$$\E[\y|x] = \begin{cases}0 & \text{if }x=(0,0)\\
\eps & \text{if }x=(1,0)\\
\eps & \text{if }x=(1,1)\\
1 & \text{if }x=(0,1).
\end{cases}$$
Let $\mC=\{x_1, x_2\}$ be the family of two classifiers based on the two coordinates. It is not hard to see that, for this ${\cal C}$ and for ${\cal H}:=\mLC$, which is the most agnostic boosting can cover, $\opt(\mH)=1/4$ is the error of the best classifier $c(x)=x_2$. On the other hand, it is not difficult to see that any multi-calibrated partition must have separate sets for  $\{(0,1)\}$ and $\{(0,0)\}$. This implies a smaller 0-1 loss of $\eps/2$. 
\end{proof}
Clearly, by reducing the constant $\eps$, the multi-calibrated loss can approach 0, but we chose not to set it to exactly 0 because standard boosting could be applied in the noiseless case.


\section{\Totals in the real valued setting}
\label{sec:real}

In this section we will consider two settings:
\begin{enumerate}
    \item {\bf The multi-class setting}. Here the labels take values in $[l]$. For each $j \in [l]$, the label $j$ is associated with the loss function $\ell(j, t)$ which is $(B, \eps)$-nice. This means that there is an interval $I_\ell$ for which the $B$-Lipshcitz and $\eps$-optimality property hold for the functions $\ell(j, t)$ for every $j \in [k]$. The loss functions for different labels could be very different, in analogy to different false positive and negative scores for the Boolean case.
    \item {\bf The real-valued setting}. Here the labels take values in $\izo$. We will assume that the loss function $\ell(y, t)$ is $B$-Lipschitz in $y$ for all $t$, and that $\ell(y, t)$ is $(B, \eps)$-nice as a function of $t$.\footnote{One can relax the Lipschitz requirement in the first argument and only require it for $t \in I$ rather than all $t$. We work with the stronger assumption for simplicity} 
\end{enumerate}

We show that for nice loss functions, omniprediction in the real-valued setting reduces to the multi-class setting.
We take $l = \lceil B/\eps \rceil$ and partition $[0,1]$ into $l$ disjoint buckets $\{b_j\}_{j=1}^l$ of width $\eps/B$ each. For $y \in b_j$, we use the loss function $\ell(j/l, t)$ in place of $\ell(y,t)$. This amounts to discretizing $y$ to $\hat{y}$ where $|y -\hat{y}| \leq \eps/B$. By the $B$-Lipschitz property, it follows that for any predictor $f:\X \rgta \R$,
\begin{align}
 \abs{\E_{\D}[\ell(\y, f(\x))] -  \E_{\D}[\ell(\hat{\y}, f(\x))]} \leq \E_{\D}\lt[ \abs{\ell(\y, f(\x) - \ell(\hat{\y}, f(\x)} \rt] \leq \eps.\label{eq:real_reduction}
\end{align}
Since $\hat{y}$ takes on $l$ discrete values, we have reduced to the multi-class setting.

Hence we now focus on generalizing our results to the multi-class setting where $\Y =[l]$ for $l \geq 2$. Recall that the definition of \mcab in the multi-class requires that for all $c \in \mC$ and $j \in [l]$, we have
\begin{align}
\label{eq:multi-class}
    \E_{\i \sim \D}\lt[\abs{\CoVar_{\D_\i}[c(x)\ind{\y = j}]}\rt] \leq \alpha. 
\end{align}

Let $\mP(l)$ denote the space of probability distributions on $l$.  Given a partition $\mS$, we  associate to each state $S_i$ a probability distribution over labels $P_i \in \mP(l)$. The canonical predictor $f^\mS: \X \rgta \mP(l)$ predicts a in $\mP$ for every $x \in \X$ and is given by 
\[ f^\mS(x) = P_i \ \forall x \in S_i. \]

In analogy with the corresponding definitions for the binary classification setting, for every $P \in \mP(l)$ we define $k_\ell(P) \in I$ to be the action that minimizes expected loss under $P$:
\begin{align*} 
k_\ell(P) = \arg\min_{t \in I}\E_{\j \sim P}[\ell(\j, t)]
\end{align*}
and the $\ell$-optimized hypothesis $h^\mS_\ell:\X \rgta I$ by
\begin{align*}
    f^\mS_\ell(x) = k_\ell(P_i) \ \forall x \in S_i.
\end{align*}

Our main theorem is a direct generalization of Theorem \ref{thm:key}. 

\begin{theorem}
\label{thm:key-k}
Let $\D$ be a distribution on $\X \times [l]$ and $\ell :[l] \times \R \rgta \R$ be a $(B, \eps)$-nice loss function. If the partition $\mS$ is $\alpha$-\mcbd\ for $\mC, \D$, then the canonical predictor $f^\mS$ is an $(\mL, \mC, k\alpha B + \eps)$-\total.     
\end{theorem}

The theorem follows from the next Lemma which generalizes Lemma \ref{lem:better} to $l \geq 2$. Since the proof follows along similar lines till the last step, we only describe the difference.

\begin{lemma}
\label{lem:better-k}
Let $\ell: [l] \times \R \rgta \R$ be an $(B, \epsilon)$ nice function.
Given $c \in \mC$, if we define the predictor $\hat{c}: \mS \rgta I_\ell$ by
\[ \hat{c}(x)  = \clip\lt(\E_{\D_i}[c(x)], I_\ell\rt) \ \text{for} \ x \in S_i. \]
then
\begin{align}
\ell_\D(\hat{c}) \leq \ell_\D(c) + l\alpha B + \eps.
\end{align}
\end{lemma}
\begin{proof}
    We follow the proof of Lemma \ref{lem:better} till Equation \eqref{eq:use-this}, this portion does not assume $\y \in \zo$. At this point we have established that
    \begin{align}
    \label{eq:used-that}
    \ell_\D(\hat{c}) - \ell_\D(c) -\eps
        &\leq B\E_{\i \sim \D}\E_{\y \sim \D_\i} \lt| \E_{\x \sim \D_i}[c(\x)] - \E_{\x \sim \D_i|\y}[c(\x)]\rt|
    \end{align}
    The following analog of Corollary \ref{cor:bin} follows from Equation \eqref{eq:covar-1} about covariance with binary random variables. For every $j \in [l]$
    \begin{align*}
        \E_{\i \sim \D} \lt[ \Pr_{\D_i}[\y = j]\lt|\E_{\D_\i|\y = j}[c(\x)] - \E_{\D_\i}[c(\x)]\rt| \rt] = \E_{\i \sim \D}\lt[\lt|\CoVar_{\D_\i}[c(\x), \ind{\y =j}\rt|\rt] \leq \alpha.
    \end{align*}
    We use this to bound the RHS of Equation \eqref{eq:used-that} as
    \begin{align*}
    \E_{\i \sim \D}\E_{\y \sim \D_\i} \lt| \E_{\x \sim \D_i}[c(\x)] - \E_{\x \sim \D_i|\y}[c(\x)]\rt| &= \E_{\i \sim \D}\sum_{j \in [l]}\Pr_{\D_\i}[\y =j] \lt| \E_{\x \sim \D_\i}[c(\x)] - \E_{\x \sim \D_\i|\y = j}[c(\x)]\rt|\\
    &= \sum_{j \in [l]} \E_{\i \sim \D}\lt[\Pr_{\D_\i}[\y =j] \lt| \E_{\x \sim \D_\i}[c(\x)] - \E_{\x \sim \D_\i|\y = j}[c(\x)]\rt|\rt]\\
    & \leq \sum_{j \in [l]}\alpha = l\alpha.
    \end{align*}
    Plugging this into the RHS of Equation \eqref{eq:used-that} gives the desired bound.
\end{proof}

We now derive the proof of Theorem \ref{thm:key-k}. From the definition of $h^\mS_\ell$ it follows that
\[ h^\mS_\ell = \argmin_{h:\mS \rgta I_\ell} \ell_\D(h). \]
Since $\hat{c}: \mS \rgta I_\ell$ this implies $ \ell_\D(h^\mS_\ell) \leq \ell_\D(\hat{c})$. The claim now follows from Lemma \ref{lem:better-k}.

\subsection{Improved bounds for squared loss}

One of the most commonly used losses for the real-valued case is the squared loss $\ell_2(y, t) = (t -y)^2$. For this loss, we can show a stronger bound. Since $\ell_2$ is  a {\em proper} scoring rule, the post-processing function $k^*_{\ell_2}$ is just the identity function, hence $h^\mS_{\ell_2}(x) = f^\mS(x) =\E_{ \y \sim \D_i}[\y] = p_i$.

We will show the following guarantee comparing it to any linear function. 

\begin{theorem}
\label{thm:ell_2}
Consider the real-valued prediction problem with $y\in [0,1]$, let $\hat{y}$ be a discretization of $y$ into $\lceil 1/\eps \rceil$ buckets (as in Equation \ref{eq:real_reduction}).
Let $\mS$ be $\alpha$-\mcbd\ for $\mC, \D$, and $f^\mS(x)$ be the predictor defined above. Let $\ell(\y, t) = (t- y)^2$ denoted the $\ell_2$ loss. 
For any function $g_w \in \mLC$ and $y \in [0,1]$, we have
\begin{align}
\label{eq:ell_2}
    \ell_{\D}(f^\mS) + \E_{(\x, \y) \sim \D} [(g_w(\x) - f^\mS(\x))^2] \leq \ell_\D(g_w)  + \lceil 1/\eps \rceil \alpha \|w\|_1 + \eps.
\end{align}
\end{theorem}


Let us compare this bound with the one implied by Corollary \ref{cor:lipschitz}, which applies since $\ell_2$ loss is $([0,1], 1,0)$-nice. The main difference is the addition of the term $\E[(g_w(\x) - f^\mS(\x))^2]$ to the LHS. This tells us that when $\lceil 1/\eps \rceil \alpha\|w\|_1$ is small (say $O(\eps)$), any linear function that is far (in squared distance) from $f^\mS$ incurs large $\ell_2$ loss. This theorem follows from the following {\em Pythagorean bound} which holds for the squared error, proved in Appendix \ref{sec:l2_proof}. 

\begin{lemma}\label{lem:l2}
In the setting of Theorem \ref{thm:ell_2},
\begin{align*}
    \abs{\E_{(\x, \y) \sim \D}[ \ell_2(\y, g_w(\x))] - \E_{(\x, \y) \sim \D} [\ell_2(\y, f_\mS(\x))] - \E_{(\x, \y) \sim \D} [(g_w(\x) - f_\mS(\x))^2] } \leq \lceil 1/\eps \rceil\alpha \|w\|_1 + \eps.
\end{align*}
\end{lemma}

\eat{
\begin{lemma}
For any $c \in \mC$, we have
\begin{align*}
        \E_{\D}\ell(\y, c(\x)) + \alpha k B \geq  \E_{\i \sim \D} \lt[\ell\lt(\y, \E_{\D_i}[c(\x)]\rt) \rt] .
\end{align*}
\end{lemma}
\begin{proof}
    By the convexity of $\ell(\y, t)$ as a function of $t$, we have
    \begin{align*}
        \E_{\D}[\ell(\y, c(\x))] = \E_{i \sim \D}\E_{(\x, \y) \sim \D_\i}[\ell(\y, c(\x))] = \E_{i \sim \D}\E_{\y \sim \D_\i}\E_{\x \sim \D_i|\y}[\ell(\y, c(\x))]\geq \E_{i \sim \D}\E_{\y \sim \D_\i}\lt[\ell\lt(\y, \E_{\x \sim\D_i|\y}[c(\x)]\rt)\rt] 
    \end{align*}
    So we can write 
    \begin{align*}
        \E_{\i \sim \D} \E_{\D_i}\lt[\ell\lt(\y, \E_{\D_i}[c(\x)]\rt)\rt] - \E_{\D}\ell(\y, c(\x)) &\leq \E_{\i \sim \D}\E_{\y \sim \D_i} \lt[ \ell\lt(\y, \E_{\x \sim\D_i}[c(\x)]\rt) - \ell\lt(\y, \E_{\x \sim\D_i|\y}[c(\x)]\rt)\rt]\\
         &\leq \E_{\i \sim \D}\E_{\y \sim \D_i}\lt[ B\abs{\E_{\x \sim\D_i}[c(\x)] - \E_{\x \sim\D_i|\y}[c(\x)]}\rt]\\
         &= \E_{\i \sim \D}\lt[\sum_{j \in [k]} B \Pr[\y = j]\abs{\E_{\x \sim\D_i}[c(\x)] - \E_{\x \sim\D_i|\y = j}[c(\x)]}\rt] \\
         & = B\sum_{j \in [k]}\E_{\i \sim \D} \lt[\Pr[\y = j]\abs{\E_{\x \sim\D_i}[c(\x)] - \E_{\x \sim\D_i|\y = j}[c(\x)]}\rt] \\
         &\leq \alpha k B
    \end{align*}
    where the last line uses Lemma \ref{lem:tech}.
\end{proof}
}

\eat{

One of the most commonly used losses for the real-valued case is the squared loss $\ell_2(y, t) = (t -y)^2$. For this loss, we can show a stronger bound. Since $\ell_2$ is  a {\em proper} scoring rule, no adjustment is needed to the canonical predictor, $f^\mS_{\ell_2}(x) = f^\mS(x) = p_i$. We will show the following guarantee comparing it to any linear function. 

\begin{theorem}
\label{thm:ell_2}
Let $\mS$ be $\alpha$-\mcbd\ for $\mC, \D$, and $f^\mS(x)$ be the predictor defined above.
For any function $g_w \in \mLC$ and $y \in [0,1]$, we have
\begin{align}
\label{eq:ell_2}
    \bar{\ell}_2(f^\mS, \D) + \E_{(\x, \y) \sim \D} [(g_w(\x) - f^\mS(\x))^2] \leq \bar{\ell_2}(g_w, \D)  +\alpha \|w\|_1 + \eps.
\end{align}
\end{theorem}


Let us compare this bound with the one implied by Corollary \ref{cor:lipschitz}, which applies since $\ell_2$ loss is $([0,1], 1,0)$-nice. The main difference is the addition of the term $\E[(g_w(\x) - f^\mS(\x))^2]$ to the LHS. This tells us that when $\alpha\|w\|_1$ is small, any linear function that is far (in squared distance) from $f^\mS$ incurs large $\ell_2$ loss. This theorem follows from the following {\em Pythagorean bound} which holds for the squared error. 

\begin{lemma}
In the setting of Theorem \ref{thm:ell_2},
\begin{align*}
    \abs{\E_{(\x, \y) \sim \D}[ \ell_2(\y, g_w(\x))] - \E_{(\x, \y) \sim \D} [\ell_2(\y, f_\mS(\x))] - \E_{(\x, \y) \sim \D} [(g_w(\x) - f_\mS(\x))^2] } \leq \alpha \|w\|_1 + \eps.
\end{align*}
\end{lemma}

\begin{proof}
Since the squared loss is $1$-Lipschitz, by the argument in Equation \eqref{eq:real_reduction}, we can work with an $\eps$-discretization of the interval $[0,1]$ with at most $\eps$ loss. Therefore, for $k=\lceil 1/\eps \rceil$, let $y=j\eps$, for $j \in \{0,\dots,k\}$.
For any $(x,y)$ we have
\begin{align*}
    (g_w(x) - y)^2 = (g_w(x) - p_i)^2 + (p_i -y)^2 + 2(g_w(x) - p_i)(p_i - y).
\end{align*}
As in the proof of Theorem \ref{thm:key}, let use denote $\alpha_i = \CoVar_{\D_i}[c(\x), \y]$. 
Fixing $i \in [m]$ and taking expectations over $(\x,\y) \sim \D_i$ 
\begin{align}
\label{eq:sq-error}
    \E_{(\x, \y) \sim \D_i}[ (g_w(\x) - \y)^2]
    &= \E_{(\x, \y) \sim \D_i} [(g_w(\x) - p_i)^2] + \E_{(\x, \y) \sim \D_i} [(p_i - \y)^2]\notag\\
    & + 2p_i\E_{(\x, \y) \sim \D_i} [p_i - \y] + 2\E_{(\x, \y) \sim \D_i} [g_w(\x)(p_i - \y )] .
\end{align}


We can simplify all except the last term  as 
\begin{align*}
    \E_{(\x, \y) \sim \D_i}[ (g_w(\x) - \y)^2] &= \E_{(\x, \y) \sim \D}[\ell_2(\y, g_w(\x)]\\
    \E_{(\x, \y) \sim \D_i} [(g_w(\x) - p_i)^2] &= \E_{(\x, \y) \sim \D_i} [(g_w(\x) - f^\mS(\x))^2]\\
    \E_{(\x, \y) \sim \D_i} [(p_i - \y)^2] &= \E_{(\x, \y) \sim \D_i} [(f^\mS(\x) - \y)^2] = \E_{(\x, \y) \sim \D_i}[\ell_2(\y, f^\mS(\x))] \\
    \E_{(\x, \y) \sim \D_i} [p_i - \y] &= 0.
\end{align*}
For the last term in Equation \eqref{eq:sq-error} we can write, 
\begin{align*}
    \E_{(\x, \y) \sim \D_i} [g_w(\x)(\y - p_i)] &= \sum_{j \geq 1}w_j \E_{(\x, \y) \sim \D_i} [c_j(\x)(\y - p_i)]\\
    &= 
\end{align*}

Plugging back into Equation \eqref{eq:sq-error} and rearranging gives  
\begin{align*}
    \abs{\E_{(\x, \y) \sim \D_i}[ \ell_2(\y, g_w(\x))] - \E_{(\x, \y) \sim \D_i} [\ell_2(\y, f^\mS(\x))] - \E_{(\x, \y) \sim \D_i} [(g_w(\x) - p_i)^2] } \leq \alpha_i \|w\|_1
\end{align*}
Now averaging over $i \in [m]$ with probability $\D(S_i)$ gives the desired inequality. 
\end{proof}

}
\section{Computing \mcbd\ partitions}
\label{sec:algo}

In this section, we show how one can use a weak agnostic learner to compute a multi-calibrated partition for the multi-class setting, where we are given a distribution $\D$ on $\X \times [l]$ and our goal is to compute a partition of $\X$ satisfying definition \ref{def:m-cal-k}. We start with the definition of a weak agnostic leaner in the setting where the class of hypotheses $\mC$ can be real-valued. Similar definitions appear in the literature in \cite{Kalai04, KalaiMV08, kk09}. We restrict the class of output hypotheses $\mC'$ found by the weak learner to be Boolean for simplicity, but this requirement is easy to relax. 

\begin{definition}
\label{def:weak}
    Let $\mC = \{c: \X \rgta \R\}$ be a family of real-valued hypotheses. 
    Let $w: (0,1] \rgta (0,1]$ be such that $w(\alpha) \leq \alpha$. 
    Let $\mC' = \{c': \X \rgta \zo\}$ be a family of Boolean hypotheses. 
    An $(w, \mC')$ weak learner for $\mC$  is given sample access to a distribution $\D$ on $\X \times \zo$. If there exists $c \in \mC$ such that $\CoVar_{\D}[c(\x), \y] \geq \alpha$,
    then with probability $1- \delta$ the weak learner returns a hypothesis $c' \in \mC'$ such that $\CoVar_{\D}[c'(\x),\y] \geq w(\alpha)$. 
\end{definition}

A few comments on our definition:
\begin{itemize}

\item We have defined the weak agnostic learner to work for every distribution $\D$ on $\X \times \zo$. One can work with a weaker notion called distribution-specific agnostic learning \cite{kk09, feldman2009distribution}, where we only have guarantees for a fixed marginal distribution on $\X$. 

\item Ideally, we would like $w(\alpha)$ to be lower bounded by a polynomial in $\alpha$, and the sample complexity and running time to depend polynomially on $1/w(\alpha), \log(1/\delta)$ and the dimensionality of $\X$. Since $w(\alpha) \leq \alpha$,  this allows a polynomial dependence in $1/\alpha$.  For simplicity, we will ignore the failure probability $\delta$ since it can be made arbitrarily small by repetition. We will state our results in terms of the number of calls made to the weak learner, through which the running time and sample complexity of our algorithm will depend on those for the weak learner.

\item  We can relax the assumption that $\mC'$ is Boolean to allow for bounded real-valued functions: if $c': \X \rgta \R$ satisfies $\CoVar[c'(\x), \y] \geq w(\alpha)$, then it is well-known \cite{Kalai04} that there exists $\theta$ such that the $\CoVar[\ind{c'(x) \geq \theta}, \y] \geq w(\alpha)/\max_{x \in \X}|c'(x)|$. Such a $\theta$ can be found by a simple line search. 
\end{itemize}

\begin{theorem}
\label{thm:alg-multi}
Given a $(w, \mC')$ weak learner for $\mC$, for any $\alpha > 0$ and any distribution $\D$ on $\X \times [l]$, there is an efficient algorithm to compute an $\alpha$-\mcbd partition for $\mC, \D$ of size 
\[ m \leq 2l\lt(\frac{12}{\alpha w(\alpha/2)^2}\rt)^{l-1}. \]
The partition can be computed by a layered branching program with nodes labelled by hypotheses from $\mC'$, of width $m$ and length $T$ where
\[ T \leq 2l\lt(\frac{12}{\alpha w(\alpha/2)^2}\rt)^l.\] 
The algorithm makes $O((l/w(\alpha/2))^{O(l)})$ queries to the weak agnostic learner. 
\end{theorem}

Our algorithm uses ideas and analyses that appear previously in the literature \cite{MansourM2002, Kalai04, gopalan2021multicalibrated}, yet since it differs each of these works along one or more axes, we  present the algorithm and sketch its analysis in Appendix \ref{app:algo} for completeness. We highlight the similarities and differences from previous work below.  

\begin{itemize}
    
    \item The notion of boosting using branching programs was introduced by \cite{MansourM2002}, building on the work of \cite{KearnsM99} on boosting via decision trees. The work of \cite{Kalai04} shows that the \cite{MansourM2002} algorithm can be used to achieve {\em correlation boosting} in real-valued setting where the labels come form $[0,1]$.
     Under the assumption that for any distribution $\D$ on $\X \times [0,1]$, the weak learner can produce a real-valued hypothesis that has non-trivial correlation with the labels; Kalai \cite{Kalai04} shows that the \cite{MansourM2002} algorithm produces a hypothesis whose correlation with the labels is close to $1$. In contrast, we work with binary labels, and only assume that the weak learner can produce a correlated hypothesis for a specific marginal distribution $\D_X$. The outcome of our algorithm is a multicalibrated partition. Our definition of multicalibration is influenced by the correlation-based notion of weak learning form this work.

    \item The notion of multicalibrated partitions was introduced in \cite{gopalan2021multicalibrated} in the context of computing importance weights. Their work may be thought of defining approximate multicalibration in the unsupervised setting, which involves some subtle technical issues. Similarly, our algorithm and analysis follows the same high level outline, but there are technical differences.  

\end{itemize}

\eat{
By the results of \cite{MansourM2002, gopalan2021multicalibrated} and \cite{hkrr2018}, there exists an $\alpha$-\mcbd partition which can be computed efficiently, given access to an efficient weak learner.

\paragraph{Algorithm for the multi-class setting.}
Similar to the binary setting, for the multi-class setting with $k$ classes there exists an $\alpha$-\mcbd partition of size $O(1/w(\alpha)^{O(k)})$ which can be computed with  $k \poly(1/w(\alpha)^k)$ queries to a weak-agnostic learner for $\mC$, by a simple modification to the algorithm of \cite{gopalan2021multicalibrated} for computing \mcbd partitions. 
Intuitively, rather than conditioning on just the predicted probability of a single label, we now condition on the entire histogram of predicted probabilities for each label in $[k]$, which gives rise to the exponential dependence on $k$. A similar result can be derived using the moment-multicalibration algorithm of \cite{Jung20}.

}

\bibliography{refs}
\appendix
\section{Relation to prior definitions of \mcab}
\label{app:def-mcab}

The work of \cite{hkrr2018, Jung20} and follow-up papers considers the setting of predictors $f:\X \rgta [0,1]$ and where $\mC$ is a collection of subsets or equivalently of Boolean functions. Our work departs from this along some axes:
\begin{enumerate}
    \item We consider partitions and not predictors. While this is a seemingly minor change, it lets us work with the covariance under the conditional distribution, which proves to be the right notion for numerous technical reasons. This is inspired by the works of \cite{gopalan2021multicalibrated} who introduced partitions for \mcab in the unsupervised setting, for the problem of computing importance weights, and  \cite{Kalai04} who uses covariance as a splitting criterion for boosting.   
    \item We allow $\mC$ to consist of arbitrary real-valued functions, not just Boolean functions. Real-valued functions have been considered for \macc \cite{kgz} but to our knowledge, not for \mcab. In the Boolean setting, our definition is essentially equivalent to that of \cite{hkrr2018, Jung20}, under a suitable setting of parameters. 
\end{enumerate}

\paragraph{Partitions versus Predictors}

A predictor is a function $f: \X \rgta [0,1]$. The goal of our algorithm is to produce a predictor that approximates the ground truth values $f^*: \X \rgta [0,1]$. We say that $f$ is $\alpha$-calibrated for a set $S$ if 
\[ \abs{\E_{\x \sim \D|S}[f(\x)] - \E_{\y \sim \D}[\y]}  \leq \alpha. \]  
One can switch between the notions of predictors and partitions with only a small loss in parameters. We first show how to derive a predictor from a partition and vice versa. 

\begin{definition}
\label{def:partition}
    For $\lambda > 0$, let $m = \lceil1/\lambda \rceil$ and let $J(\lambda) = \{J_i\}_{i=1}^m$ denote the partition of $[0,1]$ into disjoint intervals of width $\lambda$. Given a predictor $f: \X \rgta [0,1]$, the $\lambda$-canonical partition of $\X$ if given by $\mS^f =\{S_i\}_{i=1}^m$ where $S_i = f^{-1}(J_i)$. 
\end{definition}
Our notion of canonical partition is inspired by the notion of $\lambda$-discretization in \cite{hkrr2018}, however they discretize the range of $f$, rather than partition the domain $\X$. 
Suppose we start from a partition $\mS$ and let $f = f^\mS$ be its canonical predictor. The canonical $\lambda$-partion $\mS^f$ for $f$ will merge together those states in $\mS$ for which $\E[\y]$ lies in $J_i$, hence these expectations are all within $\lambda$ of each other. Hence the canonical predictor for $\mS$ and $\mS^f$  only differs from the canonical predictor for $\mS$ by $\lambda$. In the other direction, if we start from $f$ and let $\mS = \mS^f$ be the  $\lambda$-canonical partition, and $f^\mS$ to be the corresponding canonical predictor, then $f^\mS(x) = \E_{S_i}[\y]$ is simply the average of $f$ over all $x$ where $f(x) \in J_i$. If the predictor $f$ is $\alpha$-calibrated for the set $S_i$, then this average $\E_{S_i}[\y]$ is $\alpha$-close to $\E_{\x \in S_i}[f(\x)] \in J_i$, hence $|f(x) - f^\mS(x)| \leq (\alpha + \lambda)$ for any $x \in S_i$.

\paragraph{Comparing the definitions for Boolean functions}

Let us consider the family $\mC$ to consist of Boolean functions $c: \X \rgta \zo$. We can associate each $c \in \mC$ with the subset $c^{-1}(1) \subseteq \X$.
The two relevant definitions for us are the notion of \mcab\ from \cite{hkrr2018}  and mean-\mcab from \cite{Jung20}. Both definitions apply to families of sets and predictors, so we restate them specialized to the setting of a canonical predictor for a partition. \cite{Jung20} were interested in the case when $y \in [0,1]$, since their goal was to define \mcab\ for higher moments.  In our setting, we consider the Boolean case.

\begin{definition}
\cite{Jung20}
\label{def:aaron}
    Let $\mC$ be a collection of Boolean functions and let $\mS$ be a partition. The canonical predictor associated with $\mS$ is $\alpha$-mean multicalibrated if for every $c \in \mC$ 
    \begin{align}
    \label{eq:aaron}
        \Pr_{\D_i}[c(\x) =1]\abs{\E_{\D_i|c(\x) =1}[\y]  - \E_{\D_i}[\y]} \leq \alpha.
    \end{align}
\end{definition}
This is a slight variation of the original definition of \mcab\ from \cite{hkrr2018}, who required the RHS be bounded by $\alpha$, but then only consider those $i, c$ where  $\E_{\D_i}[c(\x) \geq \gamma]$. The two definitions are equivalent up to a reparametrization, the above formulation handles the case of small sets by allowing the guarantee to degrade. See Remark 2.1 of \cite{Jung20}. 

In Equation \eqref{eq:multicalibration_corollary} we condition on $\y$ and consider the expectation of $c(\x)$, whereas  \eqref{eq:aaron} conditions on $c(\x)$ and takes the expectation of $\y$. Conditioning on $\y$ extends naturally even to real-valued functions $c$ and Boolean labels $\y$, where the notion of conditioning on $c$ may not make sense. We also extend it to real-valued $\y$ by considering the event $\ind{\y \in J}$ for some interval $J$. But in fact the two definitions are equivalent when $c$ and $\y$ are both Boolean.

\begin{lemma}
\label{lem:eq-boolean}
    For Boolean functions $c$, Equation \eqref{eq:aaron} is equivalent to
    Equation \eqref{eq:m-cal}.
\end{lemma}
\begin{proof}    
Since $c$ is Boolean we have
    \begin{align*}
        \Pr_{\D_i}[c(\x) =1] &= \E_{\D_i}[c(\x)]\\
        \E_{\D_i|c(\x) =1}[\y] &= \frac{\E_{\D_i}[c(\x)y]}{\E_{\D_i}[c(\x)]}.
    \end{align*}
    We may assume $\E_{\D_i}[c(\x)] \neq 0$, else the condition holds trivially.  Multiplying both sides by $\E_{\D_i}[c(\x)]$, we can rewrite Equation \eqref{eq:aaron} as
    \begin{align*}
        \abs{\E_{\D_i}[c(\x)y] - \E_{\D_i}[c(\x)]\E_{\D_i}[y] } \leq  \alpha. 
    \end{align*}
    The quantity on the left is $|\CoVar_{\D_i}[c(\x), y]|$, hence this is identical to Equation \eqref{eq:m-cal}.   
\end{proof}

\section{Algorithm for the multi-class setting}
\label{app:algo}

In this Section, we sketch the proof of Theorem \ref{thm:alg-multi}. 

A distribution in $\mP(l)$ in a vector $v \in [0,1]^l$ such that $\sum_i v_i =1$. We can partition this space into subcubes, where each subcube is the product of $l$ intervals of length $\delta$. We refer to this collection of subcubes as $\mI_\delta$. Naively $|\mI_\delta| = O(1/\delta^l)$, we can also bound it by $l O(1/\delta^{l-1})$ by observing that since the co-ordinates need to sum to $1$, fixing the first $l-1$ co-ordinates leaves at most  $l$ possible intervals of the length $\delta$ for the last coordinate.

Our algorithm maintains a partition $\mS$ of the space $\X$.  
We have a target bound on the number of states $m = |\mI_\delta|$ for $\delta$ to be specified later. We will ensure that the number of states never exceeds $2m$. We iteratively modify the partition until we reach a multicalibrated partition, using one of two operations $\Split$ and $\Merge$. $\Split$ increases the number of states in the partition by $1$. $\Merge$  is applied whenever the number of states exceeds $2m$ and brings the number down to at most $m$. The algorithm terminates when neither operation is applicable. This results a sequence $\{\mS_t\}_{t=0}^T$ starting from the trivial partition $\mS_0 = \{ \X \}$ until the algorithm terminates with $\mS_T$, which we will show is multicalibrated. For very partition $\mS_t$, we define a corresponding distribution function $f_t : \mS_t \rgta \mP(l)$, where $f_t(S_i) \in \mP(l)$ is the distribution over labels $[l]$ conditioned on $S_i$.

\paragraph{$\Split$:} The $\Split$ operation takes a two arguments
\begin{enumerate}
    \item A state $S_i \in \mS_t$ such that $\D(S_i) \geq \alpha/4m$.
    \item A hypothesis  $c' \in \mC'$ such that there exists a label $j \in [l]$ so that 
    \[\CoVar_{\D_i}[c'(\x), \ind{\y = j}] \geq w(\alpha/2).\]
\end{enumerate}
It splits $S_i$ into two states $T_0 = S_i \cap c'^{-1}(0)$ and  $T_1 = S_i \cap c'^{-1}(1)$ to create the next partition $\mS_{t+1}$. 
To see if the $\Split$ operation can be applied, we iterate over all states $S_i$ which are sufficiently large ($\D(S_i) \geq \alpha/4m$). For every $j \in [l]$, we create a distribution $\D_{i, j}$ on $\X \times \zo$ by sampling a pair $\x, \y$ according to $\D_i$ and outputting the pair $(\x, \ind{\y =j})$. We run the weak learner on each such distribution $\D_{i, j}$. If for some $j$, it find a hypothesis $c' \in \mC'$, we run $\Split$ on the apir $(S_i, c')$.

\paragraph{$\Merge$:} The $\Merge$ operation is applied whenever the number of states exceeds $2m$, and brings it down to at most $m$. For every subcube $I \in \mI_\delta$, we merge all states $S_i$ so that $f_t(S_i) \in I$ into a single state.

\paragraph{Correctness. } Suppose the algorithm terminates at time $T$, and let $m_T$ denote the number of states. For every $i \in m_T$ such that $\D(S_i) \geq \alpha/4m$, for every $c \in \mC$ and $j \in [l]$
\begin{align}
    \label{eq:term}
    \CoVar_{\D_i}[c(\x), \ind{\y = j}] \leq \alpha/2
\end{align}
If this condition were violated, then running the weak learner on $\D_i$ is guaranteed to find $c'$, such that we can run $\Split$ on the pair $S_i, c'$, violating the termination condition. We can use these conditions to show that for any $c \in \mC$ and $j \in [l]$, 
\[ \E_{\i \sim \D}\left[\lt|\CoVar_{\D_i}[c(\x), \ind{\y =j}\rt|\right] \leq \alpha \] 
hence the partition is indeed multicalibrated.
The total contribution from {\em small} states $S_i$ such that $\D(S_i) \leq \alpha/4m$ is at most $\alpha/2$ since there are at most $2m$ states. By Equation \eqref{eq:term}, the contribution to it from large states is bounded by $\alpha/2$.

\paragraph{Running time.} We bound the number of iterations using $\E_{\x \sim \D}[\|f_t(\x)\|^2]$ as our potential function, as in \cite{Kalai04}. If $\mS_{t+1}$ is created from $\mS_t$ by splitting the state $S_i$ using $c' \in \mC'$, then one can show that
\[ \E_{\x \sim \D}[\|f_{t+1}(\x)\|^2] - \E_{\x \sim \D}[\|f_t(\x)\|^2] \geq \D(S_i) \sum_{j \in [l]} \lt(\CoVar_{\D_i}[c'(x), \ind{\y = j}]\rt)^2 \geq \frac{\alpha}{4m}w(\alpha/2)^2.\]

On the other hand, one can show that a $\Merge$ operation only causes a small reduction in the potential.  If $\mS_{t+1}$ is created from $\mS_t$ by a $\Merge$ operation, then
\[ \E_{\x \sim \D}[\|f_{t+1}(\x)\|^2] - \E_{\x \sim \D}[\|f_t(\x)\|^2] \geq -2\delta. \]

Since the $\Merge$ operation reduces the number of states from $2m$ to $m$, there are at least $m$ $\Split$ operations before the next $\Merge$ happens. If we partition time into epochs, each ending with a $\Merge$, then in each epoch, the potential increases by at least
\[ m\frac{\alpha}{4m}w(\alpha/2)^2 - 2\delta = \frac{\alpha}{4}w(\alpha/2)^2 - 2\delta \geq \delta \]
if we choose
\[ \delta = \frac{\alpha}{12}w(\alpha/2)^2. \]
Since the potential can be at most $1$, this implies a bound of $1/\delta$ on the number of epochs. We have 
\[ m_T \leq 2m \leq 2l/\delta^{l-1}  \leq 2l\lt(\frac{12}{\alpha w(\alpha/2)^2}\rt)^{l-1} \]
Since the number of epochs is at most $1/\delta$ and a single epoch involves at most $2m$ $\Split$ operations, we have
\[ T \leq 2m/\delta \leq 2l/\delta^l \leq 2l\lt(\frac{12}{\alpha w(\alpha/2)^2}\rt)^l.\]

\paragraph{Sample Complexity.}
At each time step, the weak agnostic learner might be invoked at most $l \cdot m$ times. Since we only consider those states $S_i$ where $\D(S_i) \geq \alpha/4m$, the number of samples from $\D$ needed for each call is roughly $4m/\alpha$ times the sample complexity of the weak agnostic learner. Overall, the multiplicative overhead in sample complexity over the weak agnostic learner is 
\[ l m \cdot 4m/\alpha \cdot T = O\lt(l^2 \lt(\frac{12}{\alpha w(\alpha/2)^2}\rt)^{3l}\rt). \]

\section{More Proofs}

\subsection{Proofs from Section \ref{sec:mcab}}
\label{app:mcab}

\begin{proof}[Proof of Corollary \ref{cor:covar}]
    Since $\y \in \zo$, $\E_\D[\y] = \Pr_\D[\y =1]$, while $\E_\D[\z\y] = \Pr[\y =1]\E_\D[\z|\y = 1]$. Substituting in the definition of covariance gives the first equality in 
    Equation \eqref{eq:covar-1}. To derive the second equality, we let $\y' = 1 - \y$ so that 
    \[ \CoVar_{\D}[\z, \y'] = \CoVar_{\D}[\z, 1 - \y] = -\CoVar_{\D}[\z, \y]\]
    and then apply the same reasoning to $\y'$. 
\end{proof}

\begin{proof}[Proof of Lemma \ref{lem:small}]
    Let $L \subseteq [m]$ denote the set of states for which Equation \eqref{eq:small} does not hold. For such states we use the bound
    \[ \CoVar_{\D_i}[c(\x), \y] \leq \max_{x \in \X, y \in \zo}|c(x)||y| \leq \infnorm{C} .\]
    Hence 
    \[ \sum_{i \in L}D(S_i)\CoVar_{\D_i}[c(\x), \y] \leq m\frac{\alpha}{2m\infnorm{C}}\infnorm{C} \leq \frac{\alpha}{2}.\] 
    By Equation \eqref{eq:small-cov} we have the bound
    \[ \sum_{i \not\in L}D(S_i)\CoVar_{\D_i}[c(\x), \y] \leq \frac{\alpha}{2}.
    \]
    Summing the two bounds, we conclude that $\mS$ is $\alpha$-\mcbd. 
\end{proof}

\begin{proof}[Proof of Corollary \ref{cor:bin}]
    By Equation \eqref{eq:covar-1},
    \begin{align*}
    \E_{\i \sim \D}\lt[\Pr_{\D_i}[\y = b]\lt|E_{\D_i|\y = b}[c(\x)] - \E_{\D_i}[c(\x)]\rt|\rt] = \E_{\i \sim \D}\lt[\lt|\CoVar_{\D_\i} [(c\x), \y]\rt|\rt] \leq \alpha.
    \end{align*}
\end{proof}

\begin{proof}[Proof of Lemma \ref{lem:linear}]
    Let $g_w \in \mLC(W)$.  Then for each $i \in [m]$, using the linearity of covariance
    \begin{align*}
        \CoVar_{\D_i}[g_w(\x), \y ] = w_0\CoVar[1, \y ] + \sum_{j \geq 1} w_j \CoVar_{D_i}[c_j(\x), \y)] = \sum_{j \geq 1} w_j \CoVar_{D_i}[c_j(\x), \y)]
    \end{align*}
    Averaging over states, using the triangle inequality and using \mcbn, 
    \begin{align*}
        \sum_{i \in [m]} \D(S_i) \abs{\CoVar_{\D_i}[g_w(\x), \y]} &=  \sum_{i \in [m]}\D(S_i)\abs{\sum_{j \geq 1} w_j \CoVar_{D_i}[c_j(\x), \y)] }  \\
        &\leq  \sum_{i \in [m]}\D(S_i)\sum_{j \geq 1} |w_j|\abs{\CoVar_{D_i}[c_j(\x), \y)] }  \\
        & = \sum_{j \geq 1} |w_j| \sum_{i \in [m]}\D(S_i) \abs{\CoVar_{D_i}[c_j(\x), \y)]}\\
        & \leq \alpha \sum_{ j\geq 1}|w_j| \leq \alpha W.
    \end{align*}
\end{proof}

\subsection{Proof of Theorem \ref{thm:sub-pop}}
\label{app:sub-pop}
In order to prove Theorem \ref{thm:sub-pop}, we will first restate the definition of \mcab\ as follows. For a set $S \subseteq \X$, let $S(x) =1$ if $x \in S$ and $0$ otherwise. 
\begin{lemma}
\label{lem:m-cal-newer}
The partition $\mS$ of $\X$ is $\alpha$-\mcbd for $\mC, \D$ iff for every  $c \in \mC$, 
\begin{align}
\label{eq:m-cal-newer}
    \sum_{i \in [m]}\abs{\E_\D[S_i(x) c(\x) (\y - p_i)]} \leq \alpha. 
\end{align}
\end{lemma}
\begin{proof}
    We can rewrite the LHS of Equation \eqref{eq:m-cal-new} as
    \begin{align}
    \label{eq:rewrite}
            \sum_{i \in [m]}\D(S_i)\abs{\CoVar_{\D_i}[c(\x), \y]} 
            = \sum_{i \in [m]}\D(S_i)\frac{\abs{\E_{\D}[S_i(x)c(\x)(y - p_i)]}}{\D(S_i)} = \sum_{i \in [m]}\abs{\E_{\D}[S_i(x)c(\x)(y - p_i)]}
    \end{align}
\end{proof}

\begin{proof}[Proof of Lemma \ref{thm:sub-pop}]
By Lemma \ref{lem:m-cal-newer}, our goal is to show that for any $c' \in \mC'$,
\begin{align}
    \sum_{i \in [m]}\abs{\E_{\D'}[S'_i(x) c'(\x) (\y - p'_i)]} \leq \frac{(1 + \infnorm{\mC'})\alpha}{\D(\X')}. 
\end{align}
Observe that the expectations are of quantities supported on $\X'$, and for every $x \in X'$, $\D'_i(x) = \D_i(x)/\D(\X')$, so this is equivalent to showing that 
\begin{align*}
    \sum_{i \in [m]}\abs{\E_{\D}[S'_i(x) c'(\x) (\y - p'_i)]} \leq (1 + \infnorm{\mC'})\alpha. 
\end{align*}
 We will bound the LHS using the triangle inequality as
\begin{align}
\label{eq:triangle}
    \sum_{i \in [m]}\abs{\E_{\D}[S'_i(x) c'(\x) (\y - p'_i)]} \leq     \sum_{i \in [m]}\abs{\E_{\D}[S'_i(x) c'(\x) (\y - p_i)]} +   \sum_{i \in [m]}\abs{\E_{\D}[S'_i(x) c'(\x) (p'_i- p_i)]}. 
\end{align}
    We claim that for every $x \in \X$,
    $S(x)c'(x) = S_i'(x)c'(x)$. This follows since $S_i(x) - S_i'(x) = 0$ for $x \in \X'$ and $c'(x) = 0$ for $x \not\in \X'$ since $c' \in \mC'$ is supported on $\X'$. Hence we can bound the first term by
    \begin{align}
    \label{eq:bound-1}
        \sum_{i \in [m]}\abs{\E_{\D}[S'_i(x) c'(\x) (\y - p_i)]} =  
        \sum_{i \in [m]}\abs{\E_{\D}[S_i(x) c'(\x) (\y - p_i)]} \leq \alpha.
    \end{align}
    For the second term, we apply the \mcab\ condition to the indicator function of the set $\X'$:
\begin{align*}
    \alpha \geq \sum_{i \in m}\E_{\D}[\X'(\x)S_i(\x)(\y - p_i)] = \sum_{i \in m}\abs{\E_{\D}[S'_i(\x)(\y - p_i)]} = \D(S_i')|p_i - p_i'|
\end{align*}
where we use $S_i \cap \X' = S_i'$ so $\X'(x)S_i(x) = S_i'(x)$ and
\[ p_i' = \E_\D[\y|\x \in S_i'] = \frac{\E_\D[S_i'(\x)\y]}{\D(S_i')}. \] 
We then bound the second term on the RHS of Equation \eqref{eq:triangle} by
    \begin{align}
    \label{eq:bound-2}
    \sum_{i \in [m]}\abs{\E_{\D}[S'_i(x) c(\x) (p'_i- p_i)]} &\leq \sum_{i=1}^m|p_i - p_i'| \abs{\E_{\D}[S'_i(x) c(\x)]}\notag\\
    & \leq  \sum_{i=1}^m|p_i - p_i'| \D(S_i') \infnorm{\mC'}\leq \alpha \infnorm{\mC'}.
    \end{align}
The claim now follows by plugging in Equations \eqref{eq:bound-1} and \eqref{eq:bound-2} into Equation \eqref{eq:triangle}.
\end{proof}

\subsection{Proof of Lemma \ref{lem:l2}: Stronger bounds for the $\ell_2$ loss}\label{sec:l2_proof}

\begin{proof}
Since the squared loss is $1$-Lipschitz, by the argument in Equation \eqref{eq:real_reduction}, we can work with an $\eps$-discretization of the interval $[0,1]$ with at most $\eps$ loss. Therefore, for $l=\lceil 1/\eps \rceil$, let $y=j\eps$, for $j \in \{0,\dots,l\}$.
For any $(x,y)$ we have
\begin{align*}
    (g_w(x) - y)^2 = (g_w(x) - p_i)^2 + (p_i -y)^2 + 2(g_w(x) - p_i)(p_i - y).
\end{align*}
As in the proof of Theorem \ref{thm:key}, let use denote $\alpha_i = \CoVar_{\D_i}[c(\x), \y]$. 
Fixing $i \in [m]$ and taking expectations over $(\x,\y) \sim \D_i$ 
\begin{align}
\label{eq:sq-error}
    \E_{(\x, \y) \sim \D_i}[ (g_w(\x) - \y)^2]
    &= \E_{(\x, \y) \sim \D_i} [(g_w(\x) - p_i)^2] + \E_{(\x, \y) \sim \D_i} [(p_i - \y)^2]\notag\\
    & + 2p_i\E_{(\x, \y) \sim \D_i} [p_i - \y] + 2\E_{(\x, \y) \sim \D_i} [g_w(\x)(p_i - \y )] .
\end{align}


We can simplify all except the last term  as 
\begin{align*}
    \E_{(\x, \y) \sim \D_i}[ (g_w(\x) - \y)^2] &= \E_{(\x, \y) \sim \D}[\ell_2(\y, g_w(\x)]\\
    \E_{(\x, \y) \sim \D_i} [(g_w(\x) - p_i)^2] &= \E_{(\x, \y) \sim \D_i} [(g_w(\x) - f^\mS(\x))^2]\\
    \E_{(\x, \y) \sim \D_i} [(p_i - \y)^2] &= \E_{(\x, \y) \sim \D_i} [(f^\mS(\x) - \y)^2] = \E_{(\x, \y) \sim \D_i}[\ell_2(\y, f^\mS(\x))] \\
    \E_{(\x, \y) \sim \D_i} [p_i - \y] &= 0.
\end{align*}
For the last term in Equation \eqref{eq:sq-error} we can write, 
\begin{align*}
    \E_{(\x, \y) \sim \D_i} [g_w(\x)(\y - p_i)] &= \sum_{t \geq 1}w_t \E_{(\x, \y) \sim \D_i} [c_t(\x)(\y - p_i)]\\
    &=  \sum_{t \geq 1}w_t\left[ \E_{ \y \sim \D_i} \E_{ \x \sim \D_i|\y }[c_t(\x)\y] - \E_{ (\x,\y) \sim \D_i }[c_t(\x)p_i] \right]\\
    &= \sum_{t \geq 1}w_t \left( \sum_j  \left[ \Pr[\y=j\eps] \E_{ \x \sim \D_i|\y=j\eps }[c_t(\x)j \eps]\right] -    \E_{ \x \sim \D_i }[c_t(\x)] \sum_j  \left[ \Pr[\y=j\eps] \cdot j\eps\right] \right)\\
    &= \sum_{t \geq 1}w_t \sum_j \left[  (j\eps\cdot\Pr[\y=j\eps])\left( \E_{ \x \sim \D_i|\y=j\eps }[c_t(\x)] - \E_{ \x \sim \D_i }[c_t(\x)] \right)\right]\\
    &\le \sum_{t \geq 1}w_t \sum_j (j\eps) \alpha_i \le l \alpha_i \|w\|_1
\end{align*}
where the state $S_i$ is $\alpha_i$-multicalibrated.
Plugging back into Equation \eqref{eq:sq-error} and rearranging gives  
\begin{align*}
    \abs{\E_{(\x, \y) \sim \D_i}[ \ell_2(\y, g_w(\x))] - \E_{(\x, \y) \sim \D_i} [\ell_2(\y, f^\mS(\x))] - \E_{(\x, \y) \sim \D_i} [(g_w(\x) - p_i)^2] } \leq l \alpha_i \|w\|_1.
\end{align*}
Now averaging over $i \in [m]$ with probability $\D(S_i)$ and plugging in $l=\lceil 1/\eps \rceil$ gives the desired inequality. 
\end{proof}
\eat{

\subsection{Adding label noise for agnostic boosting}
\label{app:noise}
Following Kalai-Kanade \cite{kk09} and Feldman \cite{feldman2009distribution}, we show how one can reduce \mcab\ to weak agnostic learning over a distribution with the same marginal on $\X$, by adding noise to the labels.
We define a new distribution $\D''_i$ on $\X \times \zo \times \pmo$. We first sample $(\x, \y) \sim \D_i$. We then sample $\z \in \pmo$ as a function of $\y$ so that $\E[\z|\y] = \y - p_i$ and $\z|\y$ is independent of $\x$. In other words, $\E[\z|\y =0] = -p_i$ and $\E[\z|\y = 1] = 1 -p_i$. Crucially the marginal distribution over $\X$ stays the same. 

\begin{lemma}
    Let $(\x, \z) \sim \D''_i$ and $(\x, \y) \sim \D_i$. Then 
\begin{align}
    \E_{\D''_i}[\z] &= 0,\\
    \E_{\D''_i}[c(\x) \z]  &= \E_{\D_i}[c(\x)(\y - p_i)]
\end{align} 
\end{lemma}
\begin{proof}
We have
\begin{align*} 
    \E_{\D^0_i}[\z] &= \Pr_{\D_i}[\y =0]\E_{\D^0_i}[\z|\y = 0] + \Pr_{\D_i}[\y = 1]\E_{\D^0_i}[\z|\y = 1] = (1 - p_i)(-p_i) + p_i(1 - p_i) =0.
\end{align*}
Since $\E[\z|\y] = \y - p_i$, we have
\begin{align*}
    \E_{\D^0_i}[c(\x)\z] = \E_{\D^0_i}[c(\x)\E_{D^0_i}[\z|\y, \x]] = \E_{\D_i}[c(x)(\y - p_i)]. 
\end{align*}
\end{proof}

It follows that $\mS$ is multi-calibrated for $\mC, \D$ iff
$\E_{\D''_i}[c(\x) \z] \leq \alpha p_i(1 - p_i)$. 

To compare this to distribution reweighting, note that adding label noise keeps the marginal distribution on $\X$ the same. This is an advantage if we have a {\em distribution-specific} weak learner for $\mC$. But with label noise, the weak learner has to detect a much lower correlation $\alpha p_i(1 - p_i)$, as opposed to $\alpha/2$ when we use distribution reweighting. 
}

\end{document}